\documentclass[twoside]{article}

\usepackage[accepted]{aistats2026}

\usepackage{amsmath,amssymb,amsfonts,mathtools,bm}
\usepackage{amsthm}
\usepackage{bbm}

\usepackage{algorithm}
\usepackage{algpseudocode}

\usepackage{graphicx}
\graphicspath{{Images/}}
\usepackage{caption}
\usepackage{subcaption}
\usepackage{tikz}
\usetikzlibrary{arrows.meta,positioning,calc,shapes,fit}
\usepackage{xcolor}
\usepackage{tikz}
\usetikzlibrary{calc,positioning,arrows.meta}

\newcommand{\ten}[1]{\mathcal{#1}}

\tikzset{
  tnnode/.style = {draw, circle, line width=1.2pt, fill=purple!40, minimum size=4.4mm, inner sep=0pt},
  tncore/.style = {draw, circle, line width=1.2pt, fill=purple!40, minimum size=4.8mm, inner sep=0pt},
  leg/.style   = {line width=1.2pt},
  link/.style  = {line width=1.2pt},
  inpt/.style  = {line width=1.2pt},
    sel/.style   = {draw, rectangle, line width=1.1pt, fill=orange!70,
                  minimum size=5mm, inner sep=0pt, font=\scriptsize},
  lbl/.style   = {font=\scriptsize}
}
\usepackage{stfloats} % in your preamble
\usepackage{multirow}
\usepackage{placeins}

\usepackage{booktabs}
\usepackage{makecell}  % for \thead

\usepackage{enumitem}
\setlist[itemize]{leftmargin=*, itemsep=0.2em, topsep=0.2em}

\usepackage{hyperref}
\usepackage[nameinlink,capitalize]{cleveref}
\usepackage{url}

\usepackage[round]{natbib}

\usepackage[normalem]{ulem}

\numberwithin{equation}{section}
\newtheorem{theorem}{Theorem}[section]

\newtheorem{lemma}{Lemma}[section]
\theoremstyle{definition}

\theoremstyle{remark}

\newcommand{\R}{\mathbb{R}}

\DeclareMathOperator{\Diag}{Diag}

\usepackage{xspace}
\newcommand{\methodname}{TN-SHAP\xspace}

\begin{document}

\twocolumn[

\aistatstitle{Tractable Shapley Values and Interactions via Tensor Networks}

\aistatsauthor{ Farzaneh Heidari\And Chao Li \And Guillaume Rabusseau }

\aistatsaddress{ DIRO \& Mila \\ Universit\'e de Montr\'eal  \And
RIKEN AIP \And
DIRO \& Mila, CIFAR AI Chair \\ Universit\'e de Montr\'eal } ]

\begin{abstract}
We show how to replace the $O(2^n)$ coalition enumeration  over $n$ features behind Shapley values and Shapley-style interaction indices with a \emph{few-evaluation} scheme on a tensor-network (TN) surrogate: \methodname. The key idea is to represent a predictor’s local behavior as a factorized multilinear map, so that coalitional quantities become \emph{linear probes} of a coefficient tensor. \methodname replaces exhaustive coalition sweeps with just a small number of targeted evaluations to extract order$-k$ Shapley interactions. 
In particular, both order-1 (single-feature) and order-2 (pairwise) computations have cost $O\!\big(n\,\mathrm{poly}(\chi) + n^2\big)$, where $\chi$ is the TN’s maximal cut rank. 
We provide theoretical guarantees on the approximation error and tractability of \methodname. 
On UCI datasets, our method matches enumeration on the fitted surrogate while reducing evaluation by orders of magnitude and achieves \textbf{25--1000$\times$} wall-clock speedups over KernelSHAP-IQ at comparable accuracy, while amortizing training across local cohorts.
\end{abstract}

\section{INTRODUCTION} 
Explaining \emph{how} features (individually and jointly) drive predictions is crucial for decision support, scientific discovery, and mechanistic interpretability. Shapley values~\citep{shapley:book1952} and their higher-order generalizations (Shapley–Taylor, Shapley Interaction Indices; SII~\citealp{grabisch1999axiomatic})
give principled, symmetry-respecting attributions. Yet, outside special architectures, computing these indices remains difficult in practice. 

{Existing approaches either (i) leverage structural assumptions (e.g., trees~\citep{treeshap}, linear/GAM-style models~\citep{taylor, bordt2023shapley}), (ii) rely on heavy sampling with delicate variance/weighting tuning~\citep{lundberg2017unified, fumagalli2024kernelshap}, or (iii) amortize estimation with additional training overhead~\citep{fastshap}. In this work, we follow the same general principle: we show that when a predictor can be well approximated by a low-rank multilinear tensor-network surrogate in a lifted feature space, Shapley values and Shapley-style interactions can be computed exactly and efficiently in practice.}

In this paper, we introduce \methodname: a method to compute Shapley values and Shapley-style interactions \emph{exactly on a multilinear TN function or surrogate}, using only a handful of structured evaluations. {For non-multilinear predictors, approximation quality depends on surrogate fidelity.} To reduce $O(2^n)$ complexity to $O(n)$, \methodname uses three insights:

First, we exploit the algebraic structure of multilinear functions to map coalition queries directly to tensor coefficients. Since each Shapley marginal contribution $v(C \cup \{i\}) - v(C)$ (see Section~\ref{sec:bg-shapley})   corresponds to a specific subset of monomial coefficients in the multilinear expansion, we can extract these values without explicitly evaluating all $2^n$ coalitions---the tensor representation already encodes them~\citep{Owen1972, Roth1988}.

Second, we introduce diagonal selector matrices that transform the combinatorial problem of coalition enumeration into polynomial interpolation. By augmenting feature inputs with thin diagonal matrices $S_r(t) = \text{Diag}(t, 1)$, we aggregate all coalitions of the same size into powers of $t$, allowing \methodname to recover all $n$ size-aggregated marginal contributions by solving a single $n \times n$ linear system rather than querying exponentially many subsets.

Third, we leverage tensor network factorizations to make multilinear evaluation tractable at scale. By decomposing the coefficient tensor into a TN with bond dimension $\chi$, we reduce the cost of each forward pass from $O(2^n)$ to $O(\text{poly}(\chi))$, enabling the $n$ evaluations required for interpolation to be completed in $O(n \cdot \text{poly}(\chi))$ time total, achieving polynomial rather than exponential scaling in the number of features.
Our main contributions can be summarized as follows:
\begin{itemize}[leftmargin=*, itemsep=0.25em, topsep=0.2em]
\item \textbf{Unified probe--interpolation scheme:} We develop an exact algorithm for computing Shapley values and \(k\)-way interactions through structured TN evaluations {on the surrogate function}, requiring only \(O(n)\) forward passes instead of \(O(2^n)\) coalition queries (Section~\ref{sec:method}).
\item \textbf{Feature-map–enhanced TN surrogates:} We show how learned or hand-crafted feature lifts can preserve the multilinearity required for exact computation while improving surrogate fidelity (Section~\ref{subsec:feature-maps}).
\item \textbf{Empirical validation:} On UCI benchmarks and synthetic teachers, our method achieves 25--1000$\times$ wall-clock speedups over KernelSHAP-IQ at comparable accuracy while guaranteeing exactness on the surrogate (Section~\ref{sec:experiments}).
\end{itemize}
\section{BACKGROUND AND PRELIMINARIES} \label{sec:background}
We summarize coalitional values (Shapley and Shapley-style interactions), the multilinear extension used to link them to coefficients, and the tensor/tensor-network (TN) notation we will use throughout.

\subsection{Coalitional Values and the Shapley Value} \label{sec:bg-shapley}
Let $N=\{1,\dots,n\}$ index features and let $C\subseteq N$. 
For an input $x\in\mathbb{R}^n$, a \emph{coalition value} $v(x,C)$ is the expected model output when features in $C$ are clamped to $x_C$ (restriction of $x$ to coordinates in $C$) and the complement $\bar C$ is marginalized. We consider two standard choices:
\begin{align}
\textsc{Obs:}\quad 
v_{\mathrm{obs}}(x,C)
&= \mathbb{E}_{z\sim D}\!\left[f(z)\,\middle|\, z_C = x_C\right],
\label{eq:v-obs}\\
\textsc{Int:}\quad 
v_{\mathrm{int}}(x,C)
&= \mathbb{E}_{z\sim D}\!\left[f(x_C, z_{\bar C})\right].
\label{eq:v-int}
\end{align}
where expectations are over the data distribution $D$. 
The Shapley value $\phi^{\mathrm{SV}}(i;x)$ of feature $i\in N$~\citep{shapley:book1952} is 
\begin{equation*}\scriptsize 
\sum_{s=0}^{n-1}\frac{s!(n-s-1)!}{n!} \sum_{\substack{C\subseteq N\setminus{\{i\}}\\|C|=s}} \big[v(x,C\cup{\{i\}})-v(x,C)\big].
\label{eq:sv} 
\end{equation*} 
Shapley values are weighted averages of marginal contributions where weights depend \emph{only on coalition size} $|C|$, not on specific coalition membership. This size-aggregation structure is precisely what multilinear representations used by \methodname can exploit for efficient computation.

\paragraph{Tractable Shapley values with multilinearity} \label{sec:bg-mle}
The multilinear extension (MLE) is the critical bridge between exponentially-many coalitions and polynomial-time computation. A multilinear function has the form
\begin{equation} \label{eq:mle} 
\textstyle f(x)=\sum_{T\subseteq N} c_T\prod_{j\in T} x_j,
\end{equation}
where each coefficient \(c_T\) encodes the contribution of the monomial indexed by the subset \(T\).
In a multilinear representation, each coalition query \(v(x,C)\) reads a \emph{subset of coefficients} whose indices are contained in \(C\):
setting \(x_j=1\) for \(j\in C\) and \(x_j=0\) otherwise yields
\(f(x)=\sum_{T\subseteq C} c_T\). This means:
\begin{itemize}[leftmargin=*, itemsep=0.1em, topsep=0.1em]
\item Coalition values are linear combinations of tensor coefficients $c_T$.
\item Shapley marginals $v(x,C\cup\{i\}) - v(x,C)$ correspond to specific coefficient subsets.
\item The $2^n$ coalition queries reduce to structured reads from a tensor.
\end{itemize}

Collecting the coefficients $c_T$ into an order-$n$ tensor $\ten{T}$, we can view $f$ as a multilinear map, and this tensor structure is what tensor networks will make tractable.

\subsection{Tensors and Tensor Networks}
\label{sec:bg-tensor-map}
\paragraph{Tensors as multilinear maps}
Let $\tilde x_i\in\mathbb{R}^{d_i}$ be 
$[x_i,1]^\top$.
An order-$n$ tensor $\ten T\in\mathbb{R}^{d_1\times\cdots\times d_n}$ induces a map
\begin{equation}
g(\tilde x_1,\dots,\tilde x_n)
\;=\;
\ten T \times_1 \tilde x_1 \times_2 \tilde x_2 \cdots \times_n \tilde x_n,
\label{eq:multi-map}
\end{equation}
where $\times_i$ denotes the ``tensor-vector'' product along mode $i$. See \citet{kolda2009} for background on tensors and mode-wise products. In this work, we use $g$ both as (i) a \emph{surrogate} fit to $f$ locally/globally, and (ii) an exact realization when $f$ is already multilinear. 

\begin{figure}[t]
\centering

% ---- Tensor Train (TT) ----
\begin{minipage}{0.48\linewidth}
\centering
\begin{tikzpicture}[x=0.9cm,y=0.8cm]  % smaller scale
  \node[tnnode, minimum size=6mm] (G1) at (0,0) {};
  \node[tnnode, minimum size=6mm] (G2) at (1.4,0) {};
  \node[tnnode, minimum size=6mm] (G3) at (2.8,0) {};
  \node[tnnode, minimum size=6mm] (G4) at (4.2,0) {};
  \draw[link] (G1) -- (G2) -- (G3) -- (G4);
  % physical legs (only x, no y)
  \foreach \G/\k in {G1/1, G2/2, G3/3, G4/4} {
    \draw[leg] (\G) -- ++(0,-0.7) node[lbl, below] {$\vec{x}_{\k}$};
  }
\end{tikzpicture}
\end{minipage}\hfill
%
% ---- Balanced Tensor Tree ----
\begin{minipage}{0.48\linewidth}
\centering
\begin{tikzpicture}[x=0.95cm,y=0.9cm]  % smaller scale
  \node[tncore, minimum size=6mm] (R) at (0,1.8) {};
  \node[tnnode, minimum size=5.5mm] (L) at (-1,0.9) {};
  \node[tnnode, minimum size=5.5mm] (M) at ( 1,0.9) {};
  \draw[link] (R) -- (L);
  \draw[link] (R) -- (M);

  \node[tnnode, minimum size=5.5mm] (A) at (-1.6,0) {};
  \node[tnnode, minimum size=5.5mm] (B) at (-0.4,0) {};
  \node[tnnode, minimum size=5.5mm] (C) at ( 0.4,0) {};
  \node[tnnode, minimum size=5.5mm] (D) at ( 1.6,0) {};
  \draw[link] (L) -- (A);
  \draw[link] (L) -- (B);
  \draw[link] (M) -- (C);
  \draw[link] (M) -- (D);

  \foreach \node/\k in {A/1, B/2, C/3, D/4} {
    \draw[leg] (\node) -- ++(0,-0.7) node[lbl, below] {$\vec{x}_{\k}$};
  }
\end{tikzpicture}
\end{minipage}

\caption{Two common tensor-network topologies used in this work: a tensor train (left) and a balanced binary tree (right). Only physical input legs $\vec{x}_k$ are shown.}
\label{fig:tn}
\end{figure}
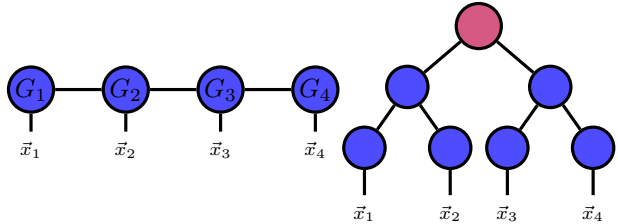

\paragraph{Tensor Networks (TN)}
\label{sec:bg-tn}
Tensor networks (TNs) represent large tensors by factorizing them into smaller tensors connected via contractions. Each node has physical legs (dangling edges, determining the tensor’s dimensions) and internal legs (edges connecting nodes), whose dimensions, called bond dimensions, control the tradeoff between expressiveness and efficiency. See Figure~\ref{fig:tn}.
A key complexity measure is the maximal cut rank $\chi$: the largest product of bond dimensions across any cut in the TN. This captures practical cost, as contraction complexity scales with the largest such cut, or equivalently, the maximal rank across bipartitions. Common TN topologies include tensor trains (TT) and balanced binary trees \citep{oseledets2011tt,cichocki2016}.

\paragraph{Contraction (forward pass).}
\emph{Contracting} a TN means summing over all internal indices until only the physical legs remain, analogous to a forward pass through the network. The runtime of such a forward pass is determined by the maximal cut rank $\chi$, and scales polynomially in $\chi$ for commonly used TN structures.  We refer to each evaluation of $g(x)$ via its TN representation as a single forward contraction.

\paragraph{TN Surrogates}
\label{sec:bg-surrogate}
We write $g$ for a TN-based surrogate of $f$ and $\ten{T}$ for the corresponding tensor. We consider both \emph{fitted surrogate}, where $\ten T$ is learned to approximate $f$ on a local/global region,
and \emph{truncated surrogate}, where $\ten T$ is a low rank TN approximation of an exact TN for $f$.

\section{METHOD}
\label{sec:method}
% \methodname  computes exact Shapley values on multilinear surrogates using three ingredients: (i) per-feature \emph{feature maps} that lift each scalar to a small channel vector while preserving overall multilinearity; (ii) per-feature \emph{selectors}, diagonal matrices, that implement toggles and interpolation weights without leaving the multilinear class; (iii) a one-parameter probing scheme whose responses form a polynomial in the probe parameter, so recovering the coefficients reduces to solving a structured linear system.
We first develop the geometric intuition behind our approach~(Sec.~\ref{subsec: geomeetric}), showing how diagonal evaluation reduces exponential queries to polynomial interpolation. We then formalize this via diagonal selectors and prove they yield the exact probe function needed for Shapley computation~(Sec.~\ref{subsec:probe}). Finally, we show how tensor networks make this tractable and how feature maps enhance surrogate fidelity~(Sec.~\ref{subsec:shapley-tn}).
\subsection{Multilinearity \& Efficient Computation}
\label{subsec: geomeetric}
To understand why these ingredients yield an $O(n)$ algorithm, consider the geometry of coalition evaluation. Computing all $2^n$ coalitions corresponds to querying the function at every vertex of the $[0,1]^n$ hypercube—an exponential task. 

Multilinear functions have special structure: they are uniquely determined by their values on any $(n+1)$ points along a diagonal line through the hypercube. Specifically, if we evaluate $f$ at points $(t, t, \ldots, t)$ for $t \in \{0, \frac{1}{n}, \frac{2}{n}, \ldots, 1\}$, the resulting polynomial $p(t) = f(t , \cdots, t)$ encodes all size-aggregated coalition values:

\begin{equation}
p(t) \;=\; f(t,\ldots,t)
= \sum_{T \subseteq N} c_T \, t^{|T|}
= \sum_{s=0}^{n} t^s \!\!\sum_{\substack{T \subseteq N \\ |T|=s}} c_T .
\end{equation}

where the coefficient of $t^s$ aggregates all coalitions of size $s$. Since Shapley values only depend on these size-aggregated marginals (see Section~\ref{sec:bg-shapley}), we can recover them by polynomial interpolation, reducing $2^n$ queries to just $n+1$.
For simplicity we present our method when all feature are mapped to 2 dimensional vectors: $\tilde x_i = [\phi(x_i), 1]$. The general case is treated in Appendix~\ref{app: feature_maps}.
\subsection{Computing Shapley Values via Diagonal Selectors}
\label{subsec:probe}
We now show how diagonal selectors enable polynomial interpolation for Shapley values.

To compute Shapley values for feature $i$, we need all marginal contributions $v(x, C \cup \{i\}) - v(x, C)$ for $C \subseteq N \setminus \{i\}$. By introducing a parameter $t$ that scales features uniformly, we can aggregate all coalitions of the same size into polynomial coefficients.

\paragraph{Diagonal Selectors}
To implement this diagonal evaluation strategy while maintaining the TN structure, we introduce \emph{selector matrices}:
\begin{equation}
S_r(t) = \text{Diag}(t, 1) = \begin{bmatrix} t & 0 \\ 0 & 1 \end{bmatrix} \in \mathbb{R}^{2 \times 2}.
\label{eq:selector-def}
\end{equation}
When applied to a lifted feature, the selector scales the data-dependent coordinate while preserving the bias coordinate:
\begin{equation}
S_r(t) \tilde{x}_r = S_r(t) \begin{bmatrix} x_r \\ 1 \end{bmatrix} = \begin{bmatrix} t \cdot x_r \\ 1 \end{bmatrix}.
\label{eq:selector-action}
\end{equation}

This elegantly implements the diagonal probe: coalitions of size $s$ naturally accumulate with weight $t^s$.
The selector $S_r(t)$ implements a soft coalition membership indicator:
\begin{itemize}[leftmargin=*, itemsep=0.1em]
\item When $t=1$: feature $r$ is fully present ($S_r(1)\tilde{x}_r = [x_r, 1]^\top$)
\item When $t=0$: feature $r$ is replaced by its constant channel ($S_r(0)\tilde{x}_r = [0, 1]^\top$)
\item For intermediate $t$: the function interpolates between these extremes
\end{itemize}
This will allow us to extract marginal contributions through polynomial interpolation.

\begin{figure}[t]
\centering

% ===================== TT with selectors (compact, single-column) =====================
\begin{minipage}{0.48\linewidth}
\centering
\begin{tikzpicture}[x=0.9cm,y=0.8cm]
  % TT cores (compact spacing)
  \node[tnnode, minimum size=5.5mm] (G1) at (0,0) {};
  \node[tnnode, minimum size=5.5mm] (G2) at (1.3,0) {};
  \node[tnnode, minimum size=5.5mm] (G3) at (2.6,0) {};
  \node[tnnode, minimum size=5.5mm] (G4) at (3.9,0) {};

  % chain links
  \draw[link] (G1) -- (G2) -- (G3) -- (G4);

  % physical legs ONLY (down) with selector then input (no upper legs)
  \foreach \G/\k in {G1/1, G2/2, G3/3, G4/4} {
    \node[sel, minimum width=7.5mm, minimum height=4.5mm] (S\k) at ($( \G )+(0,-0.75)$) {$S_{\k}(t)$};
    \draw[link] (\G) -- (S\k.north);
    \draw[inpt] (S\k.south) -- ++(0,-0.6) coordinate (xin\k);
    \node[lbl, below=1pt of xin\k] {$\vec{x}_{\k}$};
  }
\end{tikzpicture}
\end{minipage}\hfill
%
% ===================== Balanced tree with selectors (compact, single-column) =====================
\begin{minipage}{0.48\linewidth}
\centering
\begin{tikzpicture}[x=0.95cm,y=0.85cm]
  % internal nodes (compact)
  \node[tncore, minimum size=5.5mm] (R) at (0,1.7) {};
  \node[tnnode, minimum size=5mm] (L) at (-0.9,0.9) {};
  \node[tnnode, minimum size=5mm] (M) at ( 0.9,0.9) {};
  \draw[link] (R) -- (L);
  \draw[link] (R) -- (M);

  \node[tnnode, minimum size=5mm] (A) at (-1.5,0) {};
  \node[tnnode, minimum size=5mm] (B) at (-0.3,0) {};
  \node[tnnode, minimum size=5mm] (C) at ( 0.3,0) {};
  \node[tnnode, minimum size=5mm] (D) at ( 1.5,0) {};
  \draw[link] (L) -- (A);
  \draw[link] (L) -- (B);
  \draw[link] (M) -- (C);
  \draw[link] (M) -- (D);

  % selectors + inputs beneath leaves (no extra legs)
  \node[sel, minimum width=7.5mm, minimum height=4.5mm] (S1) at ($(A)+(0,-0.75)$) {$S_{1}(t)$};
  \draw[link] (A) -- (S1.north);
  \draw[inpt] (S1.south) -- ++(0,-0.6) node[lbl, below] {$\vec{x}_{1}$};

  \node[sel, minimum width=7.0mm, minimum height=4.5mm] (S2) at ($(B)+(-0.10,-0.75)$) {$S_{2}(t)$};
  \draw[link] (B) -- ++(0,-0.55) (S2.north);
  \draw[inpt] (S2.south) -- ++(0,-0.6) node[lbl, below] {$\vec{x}_{2}$};

  \node[sel, minimum width=7.1mm, minimum height=4.5mm] (S3) at ($(C)+(0.10,-0.75)$) {$S_{3}(t)$};
  \draw[link] (C) -- ++(0,-0.5) (S3.north);
  \draw[inpt] (S3.south) -- ++(0,-0.6) node[lbl, below] {$\vec{x}_{3}$};

  \node[sel, minimum width=7.5mm, minimum height=4.5mm] (S4) at ($(D)+(0,-0.75)$) {$S_{4}(t)$};
  \draw[link] (D) -- (S4.north);
  \draw[inpt] (S4.south) -- ++(0,-0.6) node[lbl, below] {$\vec{x}_{4}$};
\end{tikzpicture}
\label{fig: tn-selectors}
\end{minipage}

\caption{Two TNs with selector matrices $S_k(t)$: tensor train (left) and balanced binary tree (right).}
\label{fig: tn-selectors}
\end{figure}
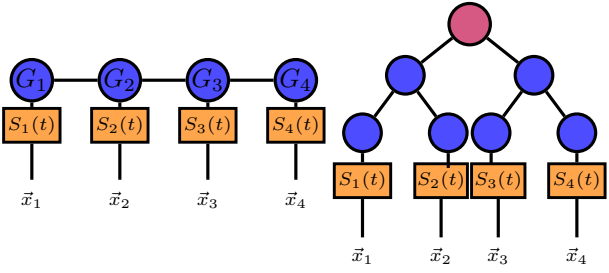

To compute the Shapley value for feature $i$, we construct a probe function $G_i(t;x)$ that aggregates all marginal contributions $v(x, C \cup \{i\}) - v(x, C)$ for $C \subseteq N \setminus \{i\}$.
For a fixed instance $x$ and target feature $i$, we define:
\begin{equation}
\label{eq:Gi-main}
\begin{aligned}
G_i(t;x) &:= g\big(M(t)S_1(t)\tilde{x}_1,\dots, M(t)S_{i-1}(t)\tilde{x}_{i-1}, S_i(1)\tilde{x}_i, \\
&\qquad M(t)S_{i+1}(t)\tilde{x}_{i+1},\dots, M(t)S_n(t)\tilde{x}_n\big) \\
&\quad - g\big(M(t)S_1(t)\tilde{x}_1,\dots, M(t)S_{i-1}(t)\tilde{x}_{i-1}, S_i(0)\tilde{x}_i, \\
&\qquad M(t)S_{i+1}(t)\tilde{x}_{i+1},\dots, M(t)S_n(t)\tilde{x}_n\big),
\end{aligned}
\end{equation}
where
$
M(t)=
\begin{bmatrix}
1 & 0\\
0 & t+1
\end{bmatrix}.
$
This applies selectors $M(t)S_r(t)$ to all features $r \neq i$ (the "complement features"), while evaluating feature $i$ in two configurations:
\begin{itemize}[leftmargin=*, itemsep=0.1em]
\item {Included}: $S_i(1)\tilde{x}_i = [x_i, 1]^\top$ (feature present)
\item {Off (bias)}: $S_i(0)\tilde{x}_i = [0, 1]^\top$ (feature absent)
\end{itemize}
The difference isolates the marginal contribution of feature $i$ across all coalitions of complement features.

% \paragraph{Why this yields a polynomial.}
Due to multilinearity, each coalition $C \subseteq N \setminus \{i\}$ with $|C| = s$ contributes a term scaled by $t^s$ (one factor of $t$ per included feature). Therefore:
\begin{equation}
\label{eq:Gi-poly}
G_i(t;x) = \sum_{s=0}^{n-1} m_s^{(i)}(x) \cdot t^s,
\end{equation}

where the coefficient $m_s^{(i)}(x) = \sum_{\substack{C \subseteq N\setminus\{i\}\\ |C|=s}} \big[v(x, C \cup \{i\}) - v(x, C)\big]$ aggregates all size-$s$ marginal contributions. {See Appendix~\ref{app:Gi-proof-diag-t1-simple} for the derivation from equations~\eqref{eq:Gi-main} to~\eqref{eq:Gi-poly}.}
{\paragraph{Relation to the Möbius transform.}
The induced set function $v(x,\cdot)$ can equivalently be represented by its Möbius coefficients $\mu(S)$~\citet{grabisch2000equivalent}, defined by $
\mu(S)=\sum_{T\subseteq S}(-1)^{|S|-|T|}\,v(x,T),
\qquad S\subseteq N.
$
In this representation, Shapley admits the classical interaction-weighted form
\begin{equation}
\label{eq: mobius}
\Phi_i^{\mathrm{SV}}(x)=\sum_{S\subseteq N:\, i\in S}\frac{1}{|S|}\,\mu(S),
\end{equation}

showing that $\Phi_i^{\mathrm{SV}}(x)$ averages all interaction terms involving feature $i$
with weight $1/|S|$. See Appendix~\ref{Mobius_connection} for details.}
\paragraph{Polynomial interpolation}
The coefficients $\{m_s^{(i)}(x)\}_{s=0}^{n-1}$ can simply and efficiently be recovered from $n$ evaluations using polynomial interpolation:
\texttt{1)}. Evaluate $G_i(t;x)$ at $n$ distinct points: $h_\ell = G_i(t_\ell; x)$ for $\ell = 0, \ldots, n-1$;
\texttt{2)}. Solve the structured linear system (also known as Vandermonde system\footnote{A linear system with a Vandermonde matrix, standard in polynomial interpolation~\citep{davis1963interpolation}.}) $Vm^{(i)} = h$, where $V_{\ell+1, r+1} = t_\ell^r$ for each $0\leq \ell,r < n$;
\texttt{3)}. Compute the Shapley values as: $\Phi_i^{\text{SV}}(x) = \sum_{s=0}^{n-1} \frac{s!(n-s-1)!}{n!} m_s^{(i)}(x)$.
The entire process, summarized in Algorithm~\ref{alg:feweval-shap}, requires only $n$ evaluations of the surrogate, exponentially fewer than the $2^n$ coalition queries needed by enumeration.
\begin{algorithm}[t]
\caption{TN-SHAP FewEval (single-feature Shapley $\Phi^{\mathrm{SV}}(i;x)$)}
\label{alg:feweval-shap}
\small
\begin{algorithmic}[1]
\Require Instance $x$, feature index $i$, probe nodes $\{t_\ell\}_{\ell=0}^{n-1}$ (all distinct)
\State Lift inputs once: $\tilde x_r \gets [\phi_r(x_r),1]$ for all $r\in N$
\For{$\ell=0$ to $n-1$} 
\Comment{Two TN forwards per node (on/off difference)}
  \State For all $r\neq i$: $\widetilde x_r^{(\ell)} \gets S_r(t_\ell)\,\tilde x_r$ with $S_r(t)=\Diag(t,1)$
  \State $\widetilde x_i^{\mathrm{on}} \gets S_i(1)\,\tilde x_i$, \quad $\widetilde x_i^{\mathrm{off}} \gets S_i(0)\,\tilde x_i$
  \State $g_1^{(\ell)} \gets g\!\big(\widetilde x_1^{(\ell)},\dots,\widetilde x_{i-1}^{(\ell)},\,\widetilde x_i^{\mathrm{on}},\,\widetilde x_{i+1}^{(\ell)},\dots,\widetilde x_n^{(\ell)}\big)$
  \State $g_0^{(\ell)} \gets g\!\big(\widetilde x_1^{(\ell)},\dots,\widetilde x_{i-1}^{(\ell)},\,\widetilde x_i^{\mathrm{off}},\,\widetilde x_{i+1}^{(\ell)},\dots,\widetilde x_n^{(\ell)}\big)$
  \State $h_\ell \gets g_1^{(\ell)} - g_0^{(\ell)}$ \Comment{$h_\ell = G_i(t_\ell; x)$}
\EndFor
\State Build Vandermonde $V\in\mathbb{R}^{n\times n}$ with $V_{\ell+1, r+1}=t_\ell^{\,r}$ for $\ell,r=0,\dots,n-1$
\State Solve $V\,m^{(i)}(x)=h$ stably (e.g., QR); denote $m^{(i)}(x)=(m^{(i)}_0,\dots,m^{(i)}_{n-1})^\top$
\State Set Shapley weights $\alpha_s=\dfrac{s!(n-s-1)!}{n!}$ for $s=0,\dots,n-1$
\State \Return $\displaystyle \Phi^{\mathrm{SV}}(i;x) \gets \sum_{s=0}^{n-1}\alpha_s\, m^{(i)}_s(x)$
\end{algorithmic}
\end{algorithm}
\subsection{Efficient Implementation with Tensor Networks}
\label{subsec:shapley-tn}

The selector-probe method developed above applies directly to tensor network surrogates, where the exponential speedup becomes practically realizable through efficient TN contractions.

\paragraph{From tensor map to TN probe.}
Let $g$ be the multilinear TN map in Eq.~\eqref{eq:multi-map} realized by a coefficient tensor $\mathcal{T}\in\mathbb{R}^{2\times\cdots\times 2}$ (binary lifts), so that $g(\tilde{x}_1,\dots,\tilde{x}_n)=\mathcal{T}\times_1\tilde{x}_1\cdots\times_n\tilde{x}_n$. 
The selector action translates naturally to TN operations: applying $S_r(t)=\text{Diag}(t,1)$ to feature $r$ means contracting the $r$-th physical leg of $\mathcal{T}$ with $S_r(t)\tilde{x}_r$ instead of $\tilde{x}_r$. This process is illustrated in Figure~\ref{fig: tn-selectors}.
\paragraph{Feature maps for enhanced surrogates}
\label{subsec:feature-maps}

While the diagonal selector method works with binary features $[x_i, 1]^\top$, real-world functions often exhibit nonlinear behavior that binary encodings cannot capture. Feature maps let us build more expressive surrogates while preserving the multilinear structure \methodname requires.

We lift each scalar feature $x_i$ to a vector $\tilde{x}_i \in \mathbb{R}^{d_i}$ via a feature map:
\begin{equation}
\phi_i: \mathbb{R} \to \mathbb{R}^{d_i-1}, \qquad \tilde{x}_i = [\phi_i(x_i), 1]^\top
\end{equation}
It is important to notice that the surrogate remains multilinear in the lifted features $(\tilde{x}_1, \ldots, \tilde{x}_n)$. This means each $\phi_i$ can be nonlinear in $x_i$, but the tensor network operates multilinearly on the resulting channels. { Importantly, the multilinearity is not necessarily in the input space: it is assumed only after applying the per-feature lift $\phi_i$, so the surrogate can still represent highly nonlinear functions of the original inputs.}

The bias coordinate provides a constant channel, allowing selectors to implement inclusion by passing $\phi(x_i)$ and exclusion by substituting the fixed value $1$, without changing the TN topology.
Many feature maps can be used, including: 
\begin{itemize}[leftmargin=*, itemsep=0.1em]
\item \text{Binary (default):} $d_i=2$, $\phi_i(x_i)=x_i$ , which is simple and efficient but has limited expressiveness
\item \text{Polynomial:} $d_i=k+1$, $\phi_i(x_i)=[x_i, x_i^2, \ldots, x_i^k]^\top$, which can capture polynomial behavior
\item \text{Fourier:} $d_i=2k + 1$,  with components $\sin(j\omega x_i)$, $\cos(j\omega x_i)$, which can  model periodic patterns
\item \text{Learned (MLP):} $d_i=r$, $\phi_i(x_i)=[\psi_i^{\text{MLP}}(x_i)]^\top$, which is expressive and can adapt to data
\end{itemize}

Richer maps improve fidelity but can increase the TN size; we observe that even modest lifts (e.g., $d_i=3$) are often effective.
\paragraph{Computational complexity.}
While the Shapley value computation through polynomial interpolation remains unchanged~($G_i(t;x)$ in  \eqref{eq:Gi-main} yields the polynomial in Eq.~\eqref{eq:Gi-poly} leading to a Vandermonde system), the TN structure dramatically reduces computational cost.
Each evaluation $G_i(t_\ell;x)$ requires only two TN forward passes (on/off toggle), hence the total cost to compute the Shapley value for feature $i$ corresponds to $2n$ TN evaluations and solving a structured $n\times n$ (Vandermonde) linear system. The overall complexity is $O(n \cdot \text{poly}(\chi) + n^2)$, where $\chi$ is the bond dimension.
This achieves the promised exponential speedup: from $O(2^n)$ coalition queries to $O(n \cdot \text{poly}(\chi))$ operations.

\section{THEORETICAL ANALYSIS}
\label{subsec:theory-main}

Our method provides three key guarantees: exactness on the surrogate, controlled error relative to the original model, and polynomial-time complexity. We state the main results here, with full proofs in Appendix~\ref{app:proofs}.

\begin{theorem}[Approximation Error for Coalitional Indices]
\label{thm:approximation}
Let $f$ be the original model and $g$ a multilinear surrogate. If the surrogate approximates coalition values uniformly well:
\[
\sup_{C\subseteq N}\!\bigl|v_g(x,C)-v_f(x,C)\bigr|\le \varepsilon,
\]
then any size-weighted coalitional index $\Phi_i = \sum_{C\subseteq N\setminus\{i\}} w(|C|,n)\,[v(x, C\cup\{i\})-v(x, C)]$ satisfies:
\begin{equation}
\bigl|\Phi_i^g-\Phi_i^f\bigr| \le 2\varepsilon \sum_{s=0}^{n-1} \bigl|w(s,n)\bigr| \binom{n-1}{s}.
\end{equation}
\end{theorem}

\noindent\textit{Interpretation:} The error in Shapley values (or other coalitional indices) is controlled by the surrogate's worst-case coalition approximation error $\varepsilon$. For Shapley values specifically, this simplifies to $|\phi_i^g - \phi_i^f| \le 2\varepsilon$ since the weights sum to 1.

\noindent\textit{Proof sketch:} Each marginal contribution has error at most $2\varepsilon$, and there are $\binom{n-1}{s}$ coalitions of size $s$. The bound follows by triangle inequality. See Appendix~\ref{app: approx_theorem} for details.

\begin{theorem}[Tractability: From Exponential to Polynomial]
\label{thm:tractability}
Let $\mathcal{T}\in\mathbb{R}^{2\times\cdots\times 2}$ be realized by a TN with maximum bond dimension $\chi$. Any size-weighted coalitional index $\Phi_i$ is computable in $O(n\cdot\text{poly}(\chi) + n^2)$ time using diagonal selector probes and polynomial interpolation.
\end{theorem}

\noindent\textit{Proof sketch:} The diagonal probe $G_i(t)$ from Eq.~\eqref{eq:Gi-main} yields polynomial $\sum_s t^s m_s^{(i)}$. Evaluating at $n$ points gives a Vandermonde system; solving recovers all size-aggregated marginals. Each TN forward costs $O(\text{poly}(\chi))$. See Appendix~\ref{app: tractability}.

\begin{table}[t]
\centering
\caption{TN-SHAP complexity by order}
\label{tab:tnshap-complexity}
\begin{tabular}{@{}lccc@{}}
\toprule
\textbf{Method} & \textbf{TN fwd} & \textbf{Solve} & \textbf{Error}\\
\midrule
SV ($k=1$) & $2n$ & $n^2$ & $2\varepsilon$ \\
Pair ($k=2$) & $4(n-1)$ & $(n-1)^2$ & $4\varepsilon$ \\
$k$-SII & $2^k(n-k+1)$ & $(n-k+1)^2$ & $2^k\varepsilon$ \\
\bottomrule
\end{tabular}
\end{table}
\paragraph{Specialization to Shapley values and interactions.}
Table~\ref{tab:tnshap-complexity} summarizes the complexity for common attribution tasks:

\begin{itemize}[leftmargin=*, itemsep=0.2em]
\item \textbf{Shapley values:} The weights $w(s,n) = \frac{s!(n-s-1)!}{n!}$ sum to 1, yielding the tight error bound $|\phi_i^g - \phi_i^f| \le 2\varepsilon$. Computing all $n$ Shapley values requires $2n^2$ TN forwards total.

\item \textbf{$k$-way interactions:} Higher-order Shapley interaction indices use inclusion-exclusion over $2^k$ subsets, requiring $2^k(n-k+1)$ TN forwards. The error grows as $2^k\varepsilon$ due to the signed sum. See Appendix~\ref{app:ksii} for details.
\end{itemize}

\paragraph{Practical runtime.}
In practice, the $O(n\cdot\text{poly}(\chi))$ complexity yields millisecond runtimes. Using Chebyshev-Gauss nodes\footnote{$t_\ell = \frac{1}{2}(1 + \cos\frac{(2\ell+1)\pi}{2m})$, optimal for polynomial interpolation stability.} and QR factorization ensures numerical robustness, while the tensor network structure keeps computations tractable through efficient contractions.

\section{EXPERIMENTS}
\label{sec:experiments}
We evaluate TN-SHAP across three settings: (1) \emph{Synthetic validation} on exactly multilinear functions (runtime-focused, exact-recovery regime); (2) \emph{Practical accuracy on real-world models} (UCI~\citep{uci_concrete, uci_diabetes, uci_energy} MLP~\citep{goodfellow2016deep} teachers) comparing TN surrogates to sampling baselines; and (3) \emph{Rank ablations \& training dynamics} on synthetic teachers to quantify capacity requirements and convergence behavior. {Throughout this section, we report three quantities: (i) teacher/model queries used to fit the surrogate, (ii) surrogate fitting time, and (iii) attribution time once the surrogate is fitted.}
\subsection{Validation on Synthetic Multilinear Functions}
We first validate on synthetic multilinear functions where ground truth is exactly computable and the multilinearity assumption holds by construction.
\paragraph{Synthetic setup.}
We generate synthetic functions from low-rank CP multilinear models
with inputs scaled to $[-1,1]$. For each dimension $d\in\{10,20,30,40,50\}$,
We draw 10 test points and, for each, fit a binary tensor-tree surrogate (rank $\chi=16$) using only the structured diagonal probes $G_i(t_\ell;x)$ at Chebyshev–Gauss nodes. 
 
Since the ground-truth functions are multilinear by
construction, once the surrogate is fitted Shapley interactions can in
principle be recovered exactly; we therefore focus here on runtime
comparisons. Table~\ref{tab:synthetic_validation} shows that TN-SHAP achieves
orders-of-magnitude faster attribution than KernelSHAP-IQ while maintaining
theoretical exactness in this setting. For $d=40$ and $d=50$, training
requires careful initialization but successful fits yield millisecond-scale
attributions.
\begin{table}[t]
\centering
\footnotesize
\setlength{\tabcolsep}{3pt}
\renewcommand{\arraystretch}{0.92}
\caption{Runtime on synthetic multilinear functions (per instance).}
\label{tab:synthetic_validation}
\begin{tabular}{@{}rcc@{}}
\toprule
Dimension & TN-SHAP (ms) & KernelSHAP-IQ (ms) \\
\midrule
10 & 3.5  & 18.7  \\
20 & 6.3  & 115.9 \\
30 & 11.1 & 194.9 \\
40 & 14.2 & 319.4 \\
50 & 19.7 & 447.3 \\
\bottomrule
\end{tabular}
\vspace{-0.35em}
\end{table}

\begin{figure}[t]
\centering

  \includegraphics[height=2.7cm]{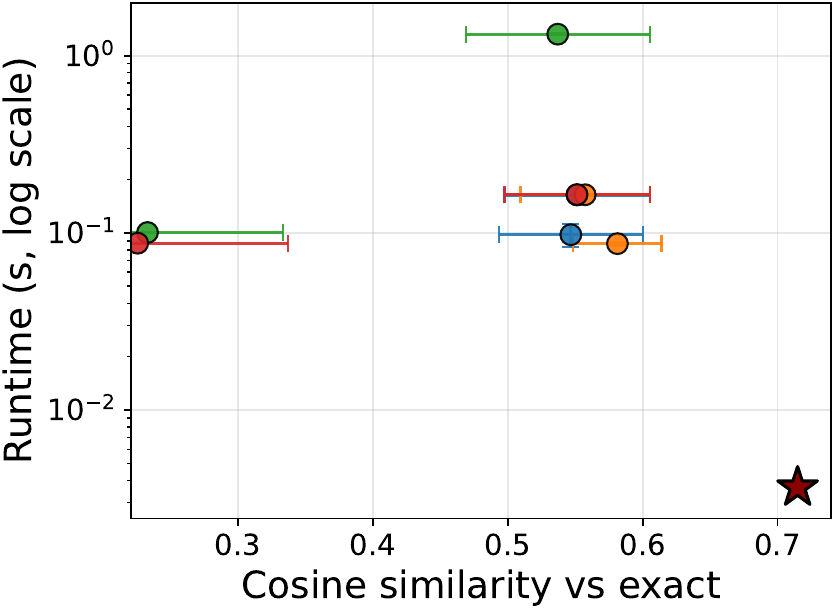}
  \includegraphics[height=2.7cm,trim={2cm 0 0 0},clip]{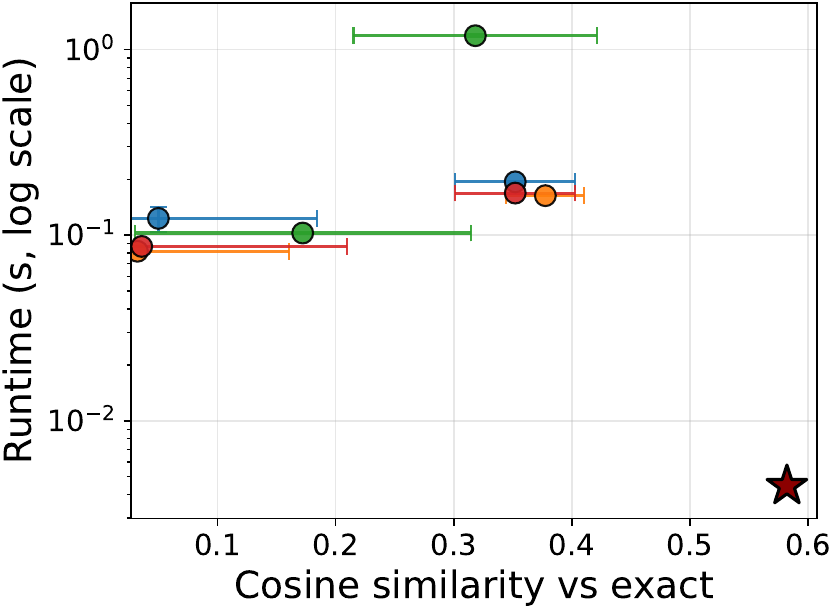}\hspace{-0.8cm}
  \raisebox{1.5cm}{\includegraphics[height=1cm]{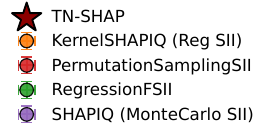}}
    \caption{Concrete: runtime (log $y$) vs.\ cosine ($x$); $k{=}2$ (left), $k{=}3$ (right). {Cosine Similarity of RegressionFSII is compared with the Faith-Shap target here.}  TN-SHAP achieves higher similarity at millisecond scale. }
  \label{fig:scatter_concrete}
\end{figure}

\subsection{Practical Accuracy on Real-World Models}
Real-world teachers (MLPs on UCI regression tasks) are only approximately multilinear, so exact recovery is not guaranteed. We therefore assess how well TN surrogates recover Shapley values and higher-order interactions in practice, comparing against sampling-based baselines under matched query budgets.

For each center test instance $\mathbf{x}_0$, we select a cohort of $100$ nearby points via $k$-NN in standardized feature space and fit a \emph{single} local binary tensor-tree surrogate (rank $\chi{=}16$) jointly on this cohort. Before the TN, each coordinate is lifted by a scalar feature map $\phi:\mathbb{R}\!\to\!\mathbb{R}$ implemented per feature with a trainable MLP with one hidden layer of 64 neurons and ReLU activation. A constant $1$ is appended as a bias term. We evaluate UCI regression tasks (Concrete, Diabetes, Energy) with 3-layer MLP teachers. The local training data around $\mathbf{x}_0$ combine: (i) a modest Gaussian neighbourhood of the cohort, $M\!\approx\!50$–$100$ samples with $\sigma=0.1\times\mathrm{std}(X_{\text{train}})$; and (ii) \textbf{$2n^2$ selector-weighted interpolation configurations}, i.e., evaluations of $G_i(t_\ell)$ for each $i\in[n]{:=}\{1,\dots,n\}$ at Chebyshev–Gauss nodes $t_\ell$. For $n{=}8$–$10$ this adds $\approx160$–$200$ structured evaluations that stabilize the polynomial interpolation. The total teacher-call budget is $M+2n^2\approx250$ per center, well below exhaustive enumeration ($2^n\approx1024$). Once fitted, we reuse the same surrogate to compute Shapley values and higher-order interactions for all $100$ points via cheap tensor contractions. We compare against KernelSHAP-IQ~\citep{fumagalli2024kernelshap} under budgets of 100–2000 queries and report accuracy against ground-truth SII obtained by exhaustive enumeration; full details appear in Appendix~\ref{app: experiment}.

\begin{table}[t]
\centering
\scriptsize
\setlength{\tabcolsep}{3pt}
\renewcommand{\arraystretch}{0.92}
\caption{\texttt{Diabetes}: mean$\pm$std cosine similarity and MSE($\times 10^{-4}$) w.r.t. exact teacher; runtime in ms. 
Best times and cosine similarity are bold; values in parentheses include amortized training cost (+65.4ms). 
SHAPIQ refers to the Regression SII variant (KernelSHAPIQ).}
\begin{tabular}{lcccc}
\toprule
Method & Budget & Time [ms] & Cos $\uparrow$ & MSE $\downarrow$ \\
\midrule
\multicolumn{5}{c}{\textbf{Order $k=1$}} \\
\midrule
\textbf{TN-SHAP} & --    & \textbf{2.8$\pm$1.2} (68.2) & \textbf{0.994$\pm$0.006} & {4.8$\pm$3.8} \\
SHAPIQ          & 100   & 87.5$\pm$9.6                & 0.990$\pm$0.009 & 7.9$\pm$4.1 \\
SHAPIQ          & 1000  & 635$\pm$214                 & 0.980$\pm$0.019 & 15.4$\pm$6.9 \\
\midrule
\multicolumn{5}{c}{\textbf{Order $k=2$}} \\
\midrule
\textbf{TN-SHAP} & --    & \textbf{3.9$\pm$0.3} (69.3) & \textbf{0.637$\pm$0.187} & 3.9$\pm$3.4 \\
SHAPIQ          & 100   & 87.5$\pm$8.9                & 0.498$\pm$0.204 & 5.7$\pm$2.5 \\
SHAPIQ          & 1000  & 633$\pm$197                 & 0.633$\pm$0.177 & 3.38$\pm$1.21 \\
\midrule
\multicolumn{5}{c}{\textbf{Order $k=3$}} \\
\midrule
\textbf{TN-SHAP} & --    & \textbf{5.1$\pm$0.4} (70.5) & {0.143$\pm$0.395} & {1.9$\pm$2.8} \\
SHAPIQ          & 100   & 178$\pm$23                  & 0.085$\pm$0.152 & 1.9$\pm$1.0 \\
SHAPIQ          & 1000  & 738$\pm$220                 & \textbf{0.175$\pm$0.209} & 2.83$\pm$0.67 \\
\bottomrule
\label{tab:diabetes_k123_concise}
\end{tabular}
\end{table}

Figure~\ref{fig:scatter_concrete} and Table~\ref{tab:diabetes_k123_concise} show a clear accuracy–efficiency advantage for TN-SHAP on the Concrete case study and diabetes dataset (averaged over 89 points, entire test set). For $k{=}1,2,3$ , TN-SHAP attains comparable cosine similarity while operating at millisecond scale, whereas SHAPIQ baselines (KernelSHAP-IQ (Reg SII)~\citep{fumagalli2024kernelshap}, SHAP-IQ (Monte Carlo)~\citep{shapiq})
 requires hundreds of milliseconds to reach comparable (often lower) accuracy. The log-scaled scatter further highlights a persistent $\sim$1–2$\times$10 speed gap at similar or better similarity. In short, the local TN surrogate delivers strong practical accuracy with dramatically reduced attribution time for higher-order interactions. 
\paragraph{Practical considerations.}
\textit{Choice of feature map.}
Across datasets, a simple \emph{binary feature map} (i.e., using $\tilde{x}_i = [x_i,1]$) works well for first–order attributions ($k{=}1$), but we observe a consistent drop in cosine similarity for higher–order interactions ($k{=}2,3$). To probe whether this is a representational bottleneck rather than an algorithmic one, we replaced the binary map with the simplest learned embedding: an MLP with one hidden layer (64 ReLU units) followed by a linear output \emph{without bias} that produces a \emph{one–dimensional} feature map. This 1D learned map significantly improves the cosine similarity for $k{=}2,3$ across our benchmarks, while leaving $k{=}1$ essentially unchanged. A broader study of higher–dimensional learned maps (and their effect on interaction orders $k{\ge}2$) is deferred to  Appendix~\ref{app: feature_map_exps}.
\\
\textit{Multilinearity.}
Neural networks are not globally multilinear, but in practice they are often \emph{locally} close enough. TN-SHAP works best in small neighborhoods around the query point, after simple feature preprocessing that smooths local behavior (our 64-unit ReLU feature maps) and when the local function has low effective interaction order. We measure this by tracking the surrogate $R^2$ in the local cohort. The method can struggle when activations saturate or the function has sharp regime changes; in those cases, we found that using richer feature maps, increasing the tensor rank (we find $\chi\!=\!15$–$20$ usually suffices), or shrinking the neighborhood solves the problem. Despite these caveats, across UCI benchmarks we find enough local multilinearity for TN-SHAP to deliver large speedups while keeping attribution quality on par with—or better than—sampling methods.

\subsection{Synthetic Experiments: Understanding Rank Requirements}
To understand how tensor rank affects approximation quality, we conduct controlled experiments on synthetic teachers where ground truth is exactly computable.
\paragraph{Setup.} We train student TNs of varying rank to approximate teachers with known structure: (1) a TN-tree teacher with rank 16, (2) a low-rank TN-tree with rank 3, and (3) an exactly multilinear function~(which would correspond to a full rank TN-tree). This allows us to isolate the effect of rank mismatch.

\begin{table}[t]
\centering
\footnotesize
\setlength{\tabcolsep}{3pt}
\renewcommand{\arraystretch}{0.92}
\caption{Rank ablation vs.\ rank-14 tensor tree teacher. First column is student's training $R^2$ and follows order-wise Shapley interactions vs exact ground-truth $R^2$.}
\label{tab:rank-ablation}
\begin{tabular}{@{}rcccc@{}}
\toprule
&&  \multicolumn{3}{c}{$R^2$ by interaction order} \\
\toprule
Rank & Train $R^2$ & $k{=}1$ & $k{=}2$ & $k{=}3$ \\
\midrule
2  & 0.704 & 0.895 & 0.853 & 0.731 \\
3  & 0.797 & 0.963 & 0.903 & 0.736 \\
4  & 0.860 & 0.964 & 0.962 & 0.911 \\
5  & 0.978 & 0.998 & 0.998 & 0.990 \\
6  & 0.991 & 0.999 & 0.997 & 0.995 \\
8  & 0.999 & 1.000 & 1.000 & 1.000 \\
10 & 1.000 & 1.000 & 1.000 & 1.000 \\
\bottomrule
\end{tabular}
\vspace{-0.4em}
\end{table}

We present here the results for the rank 16 TN-tree teacher~(results for the other teachers are given in Appendix~\ref{app:teacher-student-rank-sweep}). In this case,
Table~\ref{tab:rank-ablation} reveals a critical phase transition. With insufficient rank ($r \leq 3$), the student captures first-order effects reasonably well (SII $R^2 = 0.96$ at $r=3$) but underfits higher-order interactions (order-3 $R^2 = 0.74$). At $r=4$, order-3 accuracy jumps to $R^2 = 0.91$, and by $r=5$--$6$, all orders achieve $R^2 > 0.99$. Beyond this threshold, additional rank provides diminishing returns.

% ---- Figure 4: Training dynamics (single column, slightly larger) ----
\begin{figure}[t]
\centering
\includegraphics[width=.92\columnwidth]{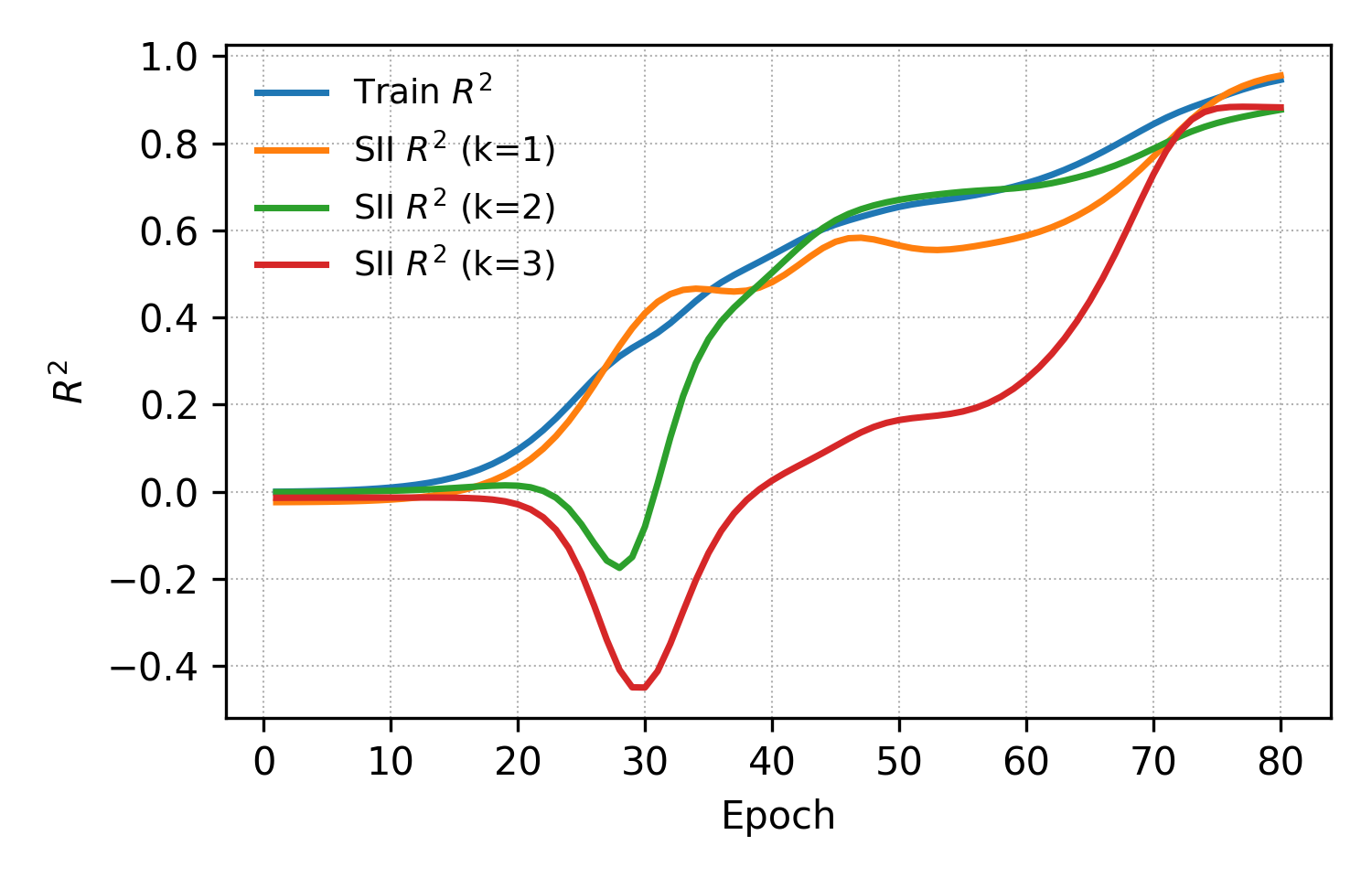}
\vspace{-0.25em}
\caption{Training dynamics: $R^2$ vs.\ epoch for the student fit (Train $R^2$) and SII ($k{=}1,2,3$).}
\label{fig:training_sweep} % (Figure 4)
\end{figure} 
Figure~\ref{fig:training_sweep} provides insight into the learning dynamics. Lower-order interactions stabilize early (order-1 plateaus by epoch 200), while higher orders continue improving. This hierarchical pattern suggests practical training strategies: monitor lower-order convergence as an early stopping criterion, and increase rank if higher orders plateau below acceptable accuracy.
TN-SHAP delivers 20--100× speedups with superior accuracy on real data and exact recovery on synthetic multilinear functions. Rank $r=5$--$6$ suffices for accurate recovery of all interaction orders, with hierarchical training dynamics providing natural convergence indicators.

{\subsection{High-dimensional Top-$k$ Recovery Experiment}
To test whether \methodname remains accurate beyond the low-dimensional settings, we test on  a higher-dimensional synthetic data following the protocol in \citep{rebuttal_mohammadi}. We evaluate three function families: degree-5 polynomial (\textsc{poly5}), degree-10 polynomial (\textsc{poly10}), and squared exponential (\textsc{sqexp}), at dimensions $D\in\{50,100\}$ using a tensor-train (TT/MPS) surrogate repeated over 100 random seeds. In each setting, we use a linear-cost query budget, corresponding to 500 teacher queries for $D=50$ and 1000 for $D=100$. We report both Top-$k$ significant-feature recovery accuracy and surrogate fidelity ($R^2$), together with surrogate training time and post-fit attribution time. As shown in Table~\ref{tab:mohammadi_highdim}, \methodname remains stable and accurate in these higher-dimensional regimes: surrogate fidelity stays high across all settings, and Top-$k$ recovery is strong even under the fixed query budget.} 
\begin{table}[t]
\centering
\setlength{\tabcolsep}{3.8pt}
\begin{tabular}{lccccc}
\toprule
Task & $D$ & Top-$K$ $\uparrow$ & $R^2$ $\uparrow$ & Train & Eval \\
\midrule
poly5  & 50  & $0.84 \pm 0.14$ & $0.95 \pm 0.05$ & 6.47  & 2.88 \\
poly5  & 100 & $0.90 \pm 0.06$ & $0.96 \pm 0.06$ & 14.73 & 3.44 \\
poly10 & 50  & $0.64 \pm 0.20$ & $0.90 \pm 0.09$ & 7.08  & 3.12 \\
poly10 & 100 & $0.76 \pm 0.15$ & $0.94 \pm 0.07$ & 14.88 & 3.50 \\
sqexp  & 50  & $0.85 \pm 0.08$ & $0.98 \pm 0.00$ & 8.51  & 3.12 \\
sqexp  & 100 & $0.98 \pm 0.02$ & $1.00 \pm 0.00$ & 25.72 & 3.15 \\
\bottomrule
\end{tabular}
\caption{\textbf{High-dimensional synthetic evaluation under a linear-cost query budget.}
 Following the high-dimensional evaluation protocol of \citet{rebuttal_mohammadi}. In each task, the target function depends on only 10\% of the input dimensions, and we report Top-$K$ significant-feature recovery accuracy, surrogate fidelity ($R^2$), surrogate training time, and attribution time (Eval).}
\label{tab:mohammadi_highdim}
\end{table}
\section{RELATED WORK}
\paragraph{SHAP computation methods}
SHAP computation spans three categories: model-specific exact methods (TreeSHAP\citep{treeshap}, DeepSHAP\citep{deeplift}), sampling-based approximations (KernelSHAP\citep{lundberg2017unified}, SAGE\citep{sage}), and amortized approaches (FastSHAP\citep{fastshap}). 
\methodname computes exact $k$-th order Shapley interactions using a small, structured set of evaluations with built-in reliability diagnostics via tensor decomposition quality, unique among current methods.

\paragraph{\textbf{Shapley values and GAMs.}}
A complementary line of work (\citep{bordt2023shapley},\citep{instashap}, \citep{park2025tensor}) shows that Shapley‑based explanations correspond to generalized additive models (GAMs) with interaction terms. \citet{bordt2023shapley} prove that \(n\)-Shapley values, Shapley–Taylor interactions \citep{taylor} and Faith‑Shap\citep{tsai2023faith} recover GAMs with interactions up to order \(n\); setting \(n{=}1\) corresponds to the classical Shapley value, which recovers an additive decomposition without interactions. \methodname bridges local Shapley explanations and global GAMs: fitting a order-$k$ surrogate yields exact $k$-SII coefficients that match the corresponding GAM terms, unifying local attributions with global additive structure.

\paragraph{Tensor networks and low-rank structure in ML.}
Prior TN methods store precomputed indices—Sobol values \citep{ballester2019sobol}, cooperative game solutions \citep{cooperativeTT}, or sparse interaction tensors \citep{constructiveTT}. While prior TN methods \citep{ballester2019sobol, cooperativeTT} use tensor decompositions to store precomputed Shapley values, we employ TNs as \emph{computational architectures}: feature maps lift inputs to a multilinear space where Shapley queries become direct tensor contractions, enabling exact $k$-SII computation in $O(2^k(n-k+1) \cdot \text{poly}(\chi))$ time without sampling or enumeration. Constructive schemes \citep{constructiveTT} provide sparse cores for interaction tensors, but still focus on representation rather than direct computation.
\paragraph{Local surrogates and faithfulness}
Local surrogate methods such as LIME \citep{ribeiro2016lime} and Anchors \citep{ribeiro2018anchors} explain individual predictions by fitting interpretable models around specific instances. However, recent work has exposed critical faithfulness issues with these approaches \citep{adebayo2018sanity,slack2020fooling,laugel2019dangers,NEURIPS2023_nxaibenchmark,shapsensitivity2025}, demonstrating that local surrogates can be manipulated or may fail to capture the true decision boundary of the underlying model. \methodname provide structural guarantees through rank bounds while computing exact SHAP values, addressing the faithfulness issues of unconstrained local methods like LIME.
\paragraph{Surrogate and kernel based Shapley methods.}
\methodname is related to methods that approximate Shapley by exploiting structure in the model. \textsc{RKHS-SHAP} derives Shapley values for kernel machines using kernel mean embeddings, avoiding Monte Carlo estimation at attribution time~\citep{chau2022rkhs_shap}. Product-kernel and fANOVA Gaussian-process methods similarly obtain exact or polynomial-time Shapley computation by leveraging strong kernel-factorization assumptions~\citep{rebuttal_mohammadi,mohammadi2025fanova}. \textsc{FourierSHAP} represents binary-input predictors through a sparse Fourier expansion and then computes SHAP in closed form ~\citep{gorji2024shap}. Regression-adjusted Monte Carlo methods combine sample reuse with a learned regression model to reduce variance~\citep{witter2025regression}. \textsc{ProxySPEX} fits structured surrogates to recover a sparse set of influential interactions efficiently~\citep{butler2025proxyspex}. These methods are complementary. \methodname targets functions that admit a low-rank multilinear tensor-network surrogate in a learned feature space.

\paragraph{\textbf{Limitations and future work}}
{The tensor network fit of \methodname  depends on the chosen tensor-network topology and, for structured factorizations such as tensor trains (TT/MPS) and tensor trees (HT/TTN), on the ordering or grouping of features. In this work, we use simple random feature orderings / groupings in our experiments, and found that this already works well empirically. Still, we do not claim that random ordering is optimal. Better strategies are possible and could further improve both approximation quality and computational efficiency, especially in higher dimensions.} There is room for further speedups and degree-wise regularization (e.g., \citep{convy2022interaction}). A complementary direction is to relate tensor–network \emph{rank} to impossibility and robustness results in explainability \citep{gunther2025informativ}, clarifying the boundary between tractable and provably intractable attribution regimes. {Another promising direction for future work is to study whether structure-exploiting surrogates such as \methodname could be useful in settings such as safety research and mechanistic interpretability, where high variance explanations are undesirable. More broadly, this viewpoint aligns with leveraging information structure instead of random sampling, as discussed in \citep{neyman2024algorithmic, neyman2025competing_sampling}.}
\section{CONCLUSION}
We presented \methodname, a tensor–network approach that turns Shapley values and Shapley-style interactions from an exponential enumeration problem into a small number of structured evaluations. By combining selector-weighted contractions with a single Vandermonde solve, \methodname reduces the cost from $O(2^n)$ coalitions to $O\!\big(n\,\mathrm{poly}(\chi)\big)$ forward passes, is exact on a multilinear surrogate, and comes with tractability and approximation guarantees. On UCI regressors, \methodname achieves millisecond-scale per-instance attributions and $25$–$1000\times$ wall-clock speedups over sampling baselines at comparable accuracy, while amortizing training across local cohorts.
\section{ACKNOWLEDGMENTS}
Guillaume Rabusseau acknowledges the support of the CIFAR AI Chair program. This work was supported by the Mitacs Globalink Research Award program and RIKEN. Chao Li was supported by the JSPS KAKENHI Grant Numbers JP24K03005. Farzaneh Heidari thanks Tensor Learning Team at RIKEN AIP for their insightful discussions and for welcoming her during her research stay at RIKEN AIP. 
\bibliographystyle{apalike}
\bibliography{references}

\clearpage
\appendix
\onecolumn
\section*{Supplementary Materials}

% ===============================
% Supplementary Preliminaries
% ===============================
% ===============================
% Supplementary Preliminaries
% ===============================
This supplement provides theoretical foundations, algorithmic details, and extended experimental results for TN-SHAP. We begin with complete proofs (Section~\ref{app:proofs}), technical details on feature maps and selectors (Section~\ref{app: feature_maps}) and algorithms for k-way interactions (Section~\ref{app:ksii}); then conclude with  comprehensive experimental protocols and additional results (Section~\ref{app: experiment}).

\section{Proofs of Main Theorems}
\label{app:proofs}
We provide complete proofs for the tractability and approximation theorems stated in Section~\ref{sec:method} of the main paper. These results establish that TN-SHAP achieves exponential speedup from $O(2^n)$ to $O(n\cdot\text{poly}(\chi) + n^2)$ while maintaining controlled approximation error. The proofs rely on the multilinear structure of the surrogate and the size-aggregation property of Shapley weights, showing how diagonal selectors enable efficient extraction of all coalition values through polynomial interpolation.
\begin{proof}[Proof of Theorem~\ref{thm:approximation}]
\label{app: approx_theorem}
Fix a center point $x$. For any model $h$, we abbreviate the coalition value function as 
$v_h(C):=v_h(x,C)$ and define the local Shapley value of feature $i$ as
\[
\Phi_i^{h}
= \sum_{C\subseteq N\setminus\{i\}} 
w(|C|,n)\bigl(v_h(C\cup\{i\}) - v_h(C)\bigr),
\]
where $w(|C|,n)=\frac{|C|!\,(n-|C|-1)!}{n!}$ are the standard Shapley weights. 
Let $v_f$ and $v_g$ denote the coalition value functions for the original target function $f$ 
and for its multilinear (TN) surrogate $g$, respectively. 
Assume that the surrogate uniformly approximates the true coalition values within $\varepsilon$, i.e.,
\[
\sup_{C\subseteq N} \bigl|v_g(C)-v_f(C)\bigr|\le \varepsilon.
\]

For a fixed feature $i\in N$, we have
\[
\begin{aligned}
\bigl|\Phi_i^g-\Phi_i^f\bigr|
&=\Bigl|\sum_{C\subseteq N\setminus\{i\}} 
w(|C|,n)\,\Bigl([v_g(C\cup\{i\})-v_f(C\cup\{i\})]
-[v_g(C)-v_f(C)]\Bigr)\Bigr|\\
&\le \sum_{C\subseteq N\setminus\{i\}} |w(|C|,n)|
\Bigl(\,|v_g(C\cup\{i\})-v_f(C\cup\{i\})|
+|v_g(C)-v_f(C)|\,\Bigr)\\
&\le 2\varepsilon \sum_{C\subseteq N\setminus\{i\}} |w(|C|,n)|.
\end{aligned}
\]
The inequality follows from the triangle inequality (or equivalently, from the Cauchy–Schwarz inequality~\citep[see, e.g.,][]{rudin1976principles,horn2012matrix}). 
Grouping subsets by their size $s=|C|$ (there are $\binom{n-1}{s}$ such coalitions) yields
\[
\bigl|\Phi_i^g-\Phi_i^f\bigr|
\;\le\;
2\varepsilon \sum_{s=0}^{n-1} |w(s,n)|\,\binom{n-1}{s},
\]
which establishes the desired bound. \qedhere
\end{proof}

\begin{proof}[Proof of Theorem~\ref{thm:tractability}]
\label{app: tractability}
Fix a center point $x\in\mathbb{R}^n$.  
For each feature $j$, define the lifted input $\tilde x_j=[x_j,\,1]^\top\in\mathbb{R}^2$ and the diagonal
selector
\[
S_j(t)=\Diag(t,\,1)\in\mathbb{R}^{2\times2},\qquad t\in[0,1].
\]
Setting $t=1$ \emph{includes} feature $j$ (uses $[x_j,1]^\top$), while $t=0$ \emph{excludes} it by
zeroing the data channel and retaining the bias, $[0,1]^\top$.  
Thus, for any coalition $C\subseteq N$, the coalition value $v_g(x,C)$ is obtained by using
$S_j(1)\tilde x_j$ for $j\in C$ and $S_j(0)\tilde x_j$ for $j\notin C$.

Let $g$ be the multilinear TN surrogate realizing $\mathcal{T}$ with maximum bond dimension $\chi$;
one forward evaluation of $g$ costs $O(\mathrm{poly}(\chi))$~\citep{levine2019quantum, oseledets2011tt, grasedyck2010hierarchical}.  Write $v_g(C):=v_g(x,C)$.

\medskip
% \textbf{Inclusion–exclusion probe.}
For the Inclusion–exclusion probe, fix $i\in N$. Define
\[
G_i(t)
:= g\!\big(\{S_j(t)\tilde x_j\}_{j\neq i},\,S_i(1)\tilde x_i\big)
 - g\!\big(\{S_j(t)\tilde x_j\}_{j\neq i},\,S_i(0)\tilde x_i\big).
\]
By multilinearity, $G_i(t)$ is a univariate polynomial of degree at most $n-1$ \emph{(here we focus on the size-$1$ case, i.e., the Shapley value for a single index $i$)}:
\[
G_i(t)=\sum_{s=0}^{n-1} m_s^{(i)}\,t^s,
\qquad
m_s^{(i)}=\!\!\sum_{\substack{C\subseteq N\setminus\{i\}\\ |C|=s}}\!\!
\bigl[v_g(C\cup\{i\})-v_g(C)\bigr].
\]
In this formulation, each factor $S_j(t)$ contributes a factor $t$ exactly when $j$ is \emph{included} in $C$; the two evaluations with $S_i(1)$ and $S_i(0)$ then implement inclusion/exclusion on $i$, producing
the Shapley marginal $v_g(C\cup\{i\})-v_g(C)$.
\medskip
% \textbf{Interpolation.}
By multilinearity, $G_i(t)$ is a univariate polynomial of degree at most $n-1$ (size-$1$ Shapley for index $i$); to recover its coefficients, we interpolate $G_i$ at $n$ distinct points $\{t_\ell\}_{\ell=0}^{n-1}\subset(0,1)$ to form the Vandermonde system

\[
V\,m^{(i)}=q,\qquad
V_{\ell+1,s+1}=t_\ell^{\,s},\ 
m^{(i)}=(m_0^{(i)},\ldots,m_{n-1}^{(i)})^\top,\ 
q_\ell=G_i(t_\ell).
\]
Each $q_\ell$ requires two forward passes (inclusion and exclusion of $i$) each costing 
$O\!\big(\mathrm{poly}(\chi)\big)$. Since $G_i$ is evaluated at $n$ interpolation points, the total probing cost is 
{$\bm{O(n\cdot\mathrm{poly}(\chi))}$}. The coefficient vector $m^{(i)}$ is then obtained by solving $V\,m^{(i)}=q$ in {$\bm{O(n^2)}$} time using a stable Vandermonde solver~\citep{bjorck1970vandermonde}.

The resulting coefficients $\{m_s^{(i)}\}_{s=0}^{n-1}$ represent the aggregated marginal 
contributions of feature $i$ grouped by coalition size, 
so the Shapley index follows directly as 
\[
\Phi_i = \sum_{s=0}^{n-1} w(s,n)\,m_s^{(i)}.
\]
Forming this weighted combination requires only $O(n)$ additional operations, 
which is negligible compared to the probing and interpolation steps. 
The total cost is therefore $\bm{O(n\cdot\mathrm{poly}(\chi) + n^2)}$, 
establishing the polynomial-time tractability claim. \qedhere

\end{proof}

% (Optional) Short remark you can keep or remove:
% \begin{remark}
% For size-only Shapley weights, the sum is $O(1)$ in $n$, giving a clean $O(\varepsilon)$ stability guarantee.
% \end{remark}

% (Signed-toggle identity remains its own section; unchanged.)

% =========================================================
% =========================================================
\section{Feature Maps and Selectors}
\label{app: feature_maps}
The feature map design is crucial for balancing expressiveness with computational tractability. Here we formalize the general $d_i$-dimensional lifting scheme sketched in the main paper, showing how various feature maps with arbitrary output dimensions can enhance surrogate fidelity (empirically in Section~\ref{app: feature_map_exps}) while preserving the multilinear structure required for exact Shapley computation. We also detail the thin diagonal selector construction and its role in maintaining constant TN topology during coalition probing.

\subsection{General Feature Map Construction}
We lift each scalar $x_i$ to $\tilde x_i=[\phi_i(x_i),\,1]^\top\in\R^{d_i}$, where $\phi_i:\R\to\R^{d_i-1}$ is a feature map providing a low-dimensional nonlinear embedding of the input. The last entry acts as a bias channel to preserve the affine closure of multilinear terms.

% \paragraph{Thin diagonal selectors.}
\subsection{Diagonal Selector Properties}

To toggle coalitions without altering TN topology, we use \emph{thin diagonal selectors}
\[
S_i(t)=\Diag(t\,I_{d_i-1},1)\in\R^{d_i\times d_i},
\]
which scale only the data-dependent channels and keep the bias channel intact:
% Lift and selector (dimension-aware)
\[
\tilde x_i=\begin{bmatrix}\phi_i(x_i)\\[2pt] 1\end{bmatrix}\in\mathbb{R}^{d_i},
\qquad
S_i(t)=\Diag\!\big(t\,I_{d_i-1},\,1\big)\in\mathbb{R}^{d_i\times d_i}.
\]

% On/off actions
\[
S_i(1)\,\tilde x_i=\begin{bmatrix}\phi_i(x_i)\\[2pt] 1\end{bmatrix},
\qquad
S_i(0)\,\tilde x_i=\begin{bmatrix}\mathbf{0}_{d_i-1}\\[2pt] 1\end{bmatrix}.
\]

This design ensures that the surrogate remains multilinear in the lifted inputs and that inclusion/exclusion operations can be expressed by simple elementwise scaling rather than structural rewiring.

% \paragraph{Complexity and stability.}
\subsection{Complexity and Stability Considerations}
Let $\chi$ denote the maximal bond dimension and $d_{\max}=\max_i d_i$. The forward cost per contraction scales as $O(\mathrm{poly}(\chi)\,d_{\max})$.  
We observe numerically that increasing $d_i$ above 3–4 offers diminishing returns for $k\!\le\!3$ interactions (see Section~\ref{app: feature_map_exps}).  

{% =========================================================
% APPENDIX (simple step-by-step proof, no jargon)
% =========================================================
\section{Tensor contraction view of \methodname}
The following lemma formally shows the equality between equations~\eqref{eq:Gi-poly} and~\eqref{eq:Gi-main}.
\label{app:Gi-proof-diag-t1-simple}
\begin{lemma}
For a fixed instance $x$ and target feature $i$, define the function $G_i(t;x)$ by
\begin{equation}
%\label{eq:Gi-main}
\begin{aligned}
G_i(t;x) &:= g\big(M(t)S_1(t)\tilde{x}_1,\dots, M(t)S_{i-1}(t)\tilde{x}_{i-1}, S_i(1)\tilde{x}_i, \\
&\qquad M(t)S_{i+1}(t)\tilde{x}_{i+1},\dots, M(t)S_n(t)\tilde{x}_n\big) \\
&\quad - g\big(M(t)S_1(t)\tilde{x}_1,\dots, M(t)S_{i-1}(t)\tilde{x}_{i-1}, S_i(0)\tilde{x}_i, \\
&\qquad M(t)S_{i+1}(t)\tilde{x}_{i+1},\dots, M(t)S_n(t)\tilde{x}_n\big),
\end{aligned}
\tag{\ref{eq:Gi-main}}
\end{equation}
where
$
M(t)=
\begin{bmatrix}
1 & 0\\
0 & t+1
\end{bmatrix}
$
and
$
S(t)=
\begin{bmatrix}
t & 0\\
0 & 1
\end{bmatrix}
$.

$G_i(t;x)$ is a polynomial in $t$ of degree at most $n-1$ given by
\begin{equation}
%\label{}
G_i(t;x)=\sum_{s=0}^{n-1} m_s^{(i)}(x)\, t^s, \tag{\ref{eq:Gi-poly}}
\end{equation}
where
\[
m_s^{(i)}(x)
=
\sum_{\substack{C\subseteq N\setminus\{i\}\\ |C|=s}}
\bigl(v(x,C\cup\{i\})-v(x,C)\bigr).
\]
\end{lemma}

\begin{proof}
Using the tensor form of the multilinear model,
\[
g(\tilde{z}_1,\dots,\tilde{z}_n)=T\times_1 \tilde{z}_1 \times_2 \tilde{z}_2 \cdots \times_n \tilde{z}_n,
\]
Equation \eqref{eq:Gi-main} becomes
\[
G_i(t;x)
=
T\times_{j\neq i} M(t)S_j(t)\tilde{x}_j \times_i S_i(1)\tilde{x}_i
-
T\times_{j\neq i} M(t)S_j(t)\tilde{x}_j \times_i S_i(0)\tilde{x}_i.
\]
Since
\[
S_i(1)\tilde{x}_i=\begin{pmatrix}x_i\\1\end{pmatrix},
\qquad
S_i(0)\tilde{x}_i=\begin{pmatrix}0\\1\end{pmatrix},
\]
and for every $j\neq i$,
\[
M(t)S_j(t)\tilde{x}_j
=
\begin{pmatrix}
t x_j\\
t+1
\end{pmatrix}
=
t\begin{pmatrix}x_j\\1\end{pmatrix}
+
\begin{pmatrix}0\\1\end{pmatrix},
\]
we obtain
\[
G_i(t;x)
=
T\times_{j\neq i}
\left(
t\begin{pmatrix}x_j\\1\end{pmatrix}
+
\begin{pmatrix}0\\1\end{pmatrix}
\right)
\times_i \begin{pmatrix}x_i\\1\end{pmatrix}
-
T\times_{j\neq i}
\left(
t\begin{pmatrix}x_j\\1\end{pmatrix}
+
\begin{pmatrix}0\\1\end{pmatrix}
\right)
\times_i \begin{pmatrix}0\\1\end{pmatrix}.
\]
By multilinearity, expanding the sum over the modes $j\neq i$ gives
\[
G_i(t;x)
=
\sum_{C\subseteq N\setminus\{i\}}
\Biggl[
T\times_{j\in C}\begin{pmatrix}x_j\\1\end{pmatrix}
\times_{j\notin C,\;j\neq i}\begin{pmatrix}0\\1\end{pmatrix}
\times_i\begin{pmatrix}x_i\\1\end{pmatrix}
-
T\times_{j\in C}\begin{pmatrix}x_j\\1\end{pmatrix}
\times_{j\notin C,\;j\neq i}\begin{pmatrix}0\\1\end{pmatrix}
\times_i\begin{pmatrix}0\\1\end{pmatrix}
\Biggr] t^{|C|}.
\]
The bracketed term is exactly
\[
v(x,C\cup\{i\})-v(x,C).
\]
Grouping terms by $s=|C|$ yields
\[
G_i(t;x)
=
\sum_{s=0}^{n-1}
\left(
\sum_{\substack{C\subseteq N\setminus\{i\}\\ |C|=s}}
\bigl(v(x,C\cup\{i\})-v(x,C)\bigr)
\right)t^s
=
\sum_{s=0}^{n-1} m_s^{(i)}(x)\, t^s,
\]
as claimed.
\end{proof}

{
\section{Connection to the Möbius transform.}
The same set function 
\label{Mobius_connection}
$v(x,\cdot):2^N\to\mathbb{R}$ also admits an equivalent expansion in terms of its Möbius coefficients~(or interaction coefficients). Define the Möbius transform of $v$ by
\begin{equation}
\mu(S):=\sum_{T\subseteq S}(-1)^{|S|-|T|}v(x,T),
\qquad S\subseteq N,
\label{eq:mobius}
\end{equation}
so that $v$ can be recovered as
\[
v(x,C)=\sum_{S\subseteq C}\mu(S).
\]
For any $C\subseteq N\setminus\{i\}$, the marginal contribution of feature $i$ is
\begin{equation}
\Delta_i v(C):=v(x,C\cup\{i\})-v(x,C)
=\sum_{S\subseteq C}\mu(S\cup\{i\}),
\label{eq:mob_fun}
\end{equation}
showing that only interaction terms involving feature $i$ contribute to $\Delta_i v(C)$.

Combining \eqref{eq:mob_fun} with the polynomial expansion of $G_i(t;x)$, we obtain
\begin{equation}
G_i(t;x)
=
\sum_{C\subseteq N\setminus\{i\}} \Delta_i v(C)\, t^{|C|},
\label{eq:mob_1}
\end{equation}
and exchanging the order of summation gives
\[
G_i(t;x)
=
\sum_{S\subseteq N\setminus\{i\}}
\mu(S\cup\{i\})
\sum_{C\supseteq S} t^{|C|}.
\]
Thus, $G_i(t;x)$ may equivalently be viewed as a polynomial whose coefficients aggregate Möbius interaction terms involving feature $i$, grouped by subset size.
Finally, summing up the marginal contribution over subset sizes recovers the standard Möbius-form expression for the Shapley value:
\begin{equation}
\Phi_i^{\mathrm{SV}}(x)
=
\sum_{\substack{S\subseteq N\\ i\in S}}
\frac{1}{|S|}\,\mu(S).
\label{eq:standard_mob}
\end{equation}
Therefore, the Shapley value averages all interaction terms containing feature $i$, with weight inversely proportional to their order.}}

\section{Order-$k$ Shapley Interactions}
\label{app:ksii}
We proceed in three steps. First, we present explicit algorithms for computing (i) pairwise ($k=2$) Shapley interaction indices~\S\ref{subsec:pairwise}and (ii) general order-$k$ indices~\citep{grabisch1999axiomatic,sundararajan2020many} using diagonal selector probes and polynomial interpolation (\S\ref{subsec:higher-order}). 
Second, we analyze the computational complexity of the general case: a direct inclusion–exclusion evaluation of $\Delta_C v_g(\cdot)$ requires $2^{k}$ forward passes per probe node, resulting in an overall cost of $O(2^{k} n\,\mathrm{poly}(\chi) + n^2)$ (\S\ref{higher_order_complexity}). 
Finally, we exploit multilinearity to derive the \emph{signed–toggle identity}
\[
\Delta_C v_g(x)
\;=\;
g\!\Big(\{S_j(t)\tilde x_j\}_{j\in N\setminus C},\,\{(S_i(1)-S_i(0))\tilde x_i\}_{i\in C}\Big),
\]

which collapses the $2^{k}$ inclusion–exclusion terms into a single evaluation per probe. This reduces the overall complexity to $O(n\,\mathrm{poly}(\chi) + n^2)$ while preserving the same interpolation-based recovery of size-aggregated coefficients~(\S\ref{subsec:signed-toggle}).

\subsection{Pairwise Interactions}
\label{subsec:pairwise}

To build intuition for our general framework, we begin with the simplest nontrivial case, $|C|=2$, where $C = \{i,j\}$. Pairwise interactions capture how the joint effect of features $i$ and $j$ differs from the sum of their individual effects~\citep{grabisch1999axiomatic}, and they illustrate the key components of our method, inclusion--exclusion, multilinearity, and polynomial interpolation, before extending to higher orders.

\paragraph{Inclusion--exclusion probe.}  
For a fixed pair $\{i,j\}$, the inclusion--exclusion principle expresses the second-order marginal contribution as a signed sum over all four inclusion patterns of the two features. For example, for $i=1$ and $j=2$
\begin{align}
Q_{\{i,j\}}(t;x) = 
% g(\dots,S_i(1)\tilde x_i,S_j(1)\tilde x_j,\dots) 
% - g(\dots,S_i(1)\tilde x_i,S_j(0)\tilde x_j,\dots) \nonumber \\
% &\quad - g(\dots,S_i(0)\tilde x_i,S_j(1)\tilde x_j,\dots) 
% + g(\dots,S_i(0)\tilde x_i,S_j(0)\tilde x_j,\dots). 
&g(S_i(1)\tilde x_i,S_j(1)\tilde x_j,S_3(t),\dots,S_n(t)) 
- g(S_i(1)\tilde x_i,S_j(0)\tilde x_j,S_3(t),\dots,S_n(t)) \nonumber \\
-&g(S_i(0)\tilde x_i,S_j(1)\tilde x_j,S_3(t),\dots,S_n(t))
 + g(S_i(0)\tilde x_i,S_j(0)\tilde x_j,S_3(t),\dots,S_n(t)).
\label{eq:pairwise-probe}
\end{align}
Each term corresponds to one configuration of feature inclusion, and the alternating signs implement a discrete second-order finite difference over the lifted inputs. By multilinearity of $g$, $Q_{\{i,j\}}(t;x)$ is a polynomial in $t$ of degree at most $n-2$:
\[
Q_{\{i,j\}}(t;x) = \sum_{s=0}^{n-2} c^{(ij)}_s(x)\, t^s,
\]
where $c^{(ij)}_s(x)$ aggregates all coalitional contributions of size $s$ among the remaining $n-2$ features.

\paragraph{Polynomial interpolation.}  
Evaluating \eqref{eq:pairwise-probe} at $n-1$ distinct points $\{t_\ell\}_{\ell=0}^{n-2}\subset(0,1)$ gives the Vandermonde system
\[
V\,c^{(ij)} = q, \qquad V_{\ell+1,s+1} = t_\ell^{\,s}, \quad q_\ell = Q_{\{i,j\}}(t_\ell;x),
\]
whose solution $c^{(ij)}$ recovers all size-aggregated coefficients. These coefficients encode how the rest of the features interact jointly with the pair $\{i,j\}$.

\paragraph{Recovering the interaction index.}  
The pairwise Shapley--Taylor (or SII) interaction index~\citep{grabisch1999axiomatic, sundararajan2020many} follows as a weighted sum of the coefficients:
\begin{equation}
\Phi^{\mathrm{SII}}(\{i,j\};x) = \sum_{s=0}^{n-2} \beta_s(n,2)\, c^{(ij)}_s(x),
\qquad
\beta_s(n,2) = \frac{s!\,(n-2-s)!}{(n-1)!}.
\label{eq:pairwise-sii}
\end{equation}
The weights $\beta_s(n,2)$ ensure the correct combinatorial normalization across coalition sizes.

% \paragraph{Algorithmic implementation.}  
Algorithm~\ref{alg:pairwise-sii} implements this procedure end-to-end: it builds lifted inputs, evaluates the inclusion--exclusion probe at $n-1$ interpolation points, solves the Vandermonde system, and combines the coefficients to produce $\Phi^{\mathrm{SII}}(\{i,j\};x)$. This pairwise case illustrates the fundamental computational structure of our method and directly generalizes to higher-order feature interactions.
% =========================================================
% Pairwise SII (order-2) via selector probes + interpolation
% =========================================================
% \subsection{Pairwise SII Algorithm}
\begin{algorithm}[H]
\caption{TN-SHAP Pairwise Interaction \(\Phi^{\mathrm{SII}}(\{i,j\};x)\)}
\label{alg:pairwise-sii}
\small
\begin{algorithmic}[1]
\Require Instance \(x\in\R^n\); feature pair \((i,j)\); distinct probe points \(\{t_\ell\}_{\ell=0}^{n-2}\)
\Ensure Pairwise SII \(\Phi^{\mathrm{SII}}(\{i,j\};x)\)

\State \textbf{Lift once:} \(\tilde x_r \gets [\phi_r(x_r),1]\) for all \(r\in N\)
\For{\(\ell = 0 \ \textbf{to}\ n-2\)} \Comment degree \(\le n-2\) \(\Rightarrow\) \(n{-}1\) points
  \State \(\tilde x_{r\notin\{i,j\}}^{(\ell)} \gets S_r(t_\ell)\,\tilde x_r\)
  \State \(\tilde x_i^{(1)}\!\gets\!S_i(1)\tilde x_i,\quad \tilde x_i^{(0)}\!\gets\!S_i(0)\tilde x_i\)
  \State \(\tilde x_j^{(1)}\!\gets\!S_j(1)\tilde x_j,\quad \tilde x_j^{(0)}\!\gets\!S_j(0)\tilde x_j\)
  \State \(q_\ell \gets \ \ g(\dots,\tilde x_i^{(1)},\tilde x_j^{(1)},\dots)\)
         \(-\ g(\dots,\tilde x_i^{(1)},\tilde x_j^{(0)},\dots)\)
         \(-\ g(\dots,\tilde x_i^{(0)},\tilde x_j^{(1)},\dots)\)
         \(+\ g(\dots,\tilde x_i^{(0)},\tilde x_j^{(0)},\dots)\)
  \Comment 4 TN forwards
\EndFor
\State Build Vandermonde \(V\in\R^{(n-1)\times(n-1)}\), \(V_{\ell+1,s+1}=t_\ell^{\,s}\)
\State Solve \(V\,c^{(ij)}=q\) (e.g., QR) to obtain coefficients \(c^{(ij)}_s,\ s=0{:}n-2\)
\State \textbf{Combine:} \(\displaystyle \Phi^{\mathrm{SII}}(\{i,j\};x)=\sum_{s=0}^{n-2}\beta_s(n,2)\,c^{(ij)}_s\)
\State \Return \(\Phi^{\mathrm{SII}}(\{i,j\};x)\)
\end{algorithmic}
\textbf{Complexity:} \(4(n{-}1)\) TN forwards \(+\ O((n{-}1)^2)\) for the solve; each forward is \(O(\mathrm{poly}(\chi))\).
\end{algorithm}
\paragraph{Complexity.}  
The procedure requires $4(n-1)$ forward evaluations of the tensor network (since $2^2=4$ inclusion patterns) and $O((n-1)^2)$ time for solving the Vandermonde system. This cost is polynomial in $n$ and the bond dimension $\chi$, and it reveals the essential tradeoff that generalizes to the $2^k$ scaling for higher-order interactions.

\subsection{Higher-Order Feature Interactions}
\label{subsec:higher-order}
The same principles underlying the pairwise case extend naturally to any coalition $C \subseteq N$ of size $|C| = k$. Higher-order interactions quantify how the joint effect of a group of $k$ features differs from the sum of their contributions taken in smaller coalitions~\citep{grabisch1999axiomatic}, capturing nonlinear and synergistic dependencies that arise only when these features act together.

\paragraph{General inclusion--exclusion probe.}  
For an arbitrary coalition $C$ with $|C| = k$, the $k$-th order marginal contribution is defined as the alternating sum over all $2^k$ inclusion patterns:
\begin{equation}
Q_C(t;x) \;=\; 
\sum_{\mathbf{b} \in \{0,1\}^k} 
(-1)^{k - \|\mathbf{b}\|_1}\;
g\Big(
\{S_r(t)\tilde x_r\}_{r \notin C}, 
\{S_i(b_i)\tilde x_i\}_{i \in C}
\Big),
\label{eq:Q-general}
\end{equation}
where $\mathbf{b} = (b_i)_{i\in C}$ is a binary vector indicating inclusion ($b_i=1$) or exclusion ($b_i=0$) of each feature $i$ in $C$. The signs implement inclusion--exclusion over $C$, while the selectors $S_r(t)$ modulate the contribution of the remaining $n - k$ features. By multilinearity, $Q_C(t;x)$ is a univariate polynomial in $t$ of degree at most $n - k$:
\[
Q_C(t;x) = \sum_{s=0}^{n-k} c^{(C)}_s(x)\, t^s,
\]
where each coefficient $c^{(C)}_s(x)$ aggregates the contributions of all coalitions of size $s$ among the remaining features.

\paragraph{Interpolation and coefficient recovery.}  
Evaluating $Q_C(t;x)$ at $n-k+1$ distinct probe points $\{t_\ell\}_{\ell=0}^{n-k}\subset (0,1)$ gives the Vandermonde system
\[
V\,c^{(C)} = q, \qquad 
V_{\ell+1,s+1} = t_\ell^{\,s}, \quad 
q_\ell = Q_C(t_\ell;x),
\]
whose solution $c^{(C)} = (c^{(C)}_0, \dots, c^{(C)}_{n-k})^\top$ recovers all size-grouped contributions. 

\paragraph{Interaction index reconstruction.}  
The order-$k$ Shapley--Taylor (or SII) interaction index~\citep{grabisch1999axiomatic,sundararajan2020many} follows as a size-weighted linear combination of these coefficients:
\begin{equation}
\Phi^{\mathrm{SII}}(C;x) = 
\sum_{s=0}^{n-k} \beta_s(n,k)\, c^{(C)}_s(x),
\qquad 
\beta_s(n,k) = \frac{s!\,(n-k-s)!}{(n-k+1)!}.
\label{eq:general-sii}
\end{equation}
This expression exactly recovers the canonical multilinear extension of the SII under size-based weighting.
% \paragraph{Algorithmic formulation.}  
Algorithm~\ref{alg:kway-sii} implements the full procedure: for a target coalition $C$, it constructs the inclusion--exclusion probe~\eqref{eq:Q-general}, evaluates it at $n - k + 1$ probe points, solves the Vandermonde system, and combines the recovered coefficients via~\eqref{eq:general-sii}. This general routine reduces to the pairwise algorithm when $k = 2$.

\begin{algorithm}[H]
\caption{TN-SHAP \(k\)-Way Interaction \(\Phi^{\mathrm{SII}}(S;x)\) for \(S\subseteq N,\ |S|=k\)}
\label{alg:kway-sii}
\small
\begin{algorithmic}[1]
\Require Instance \(x\in\R^n\); target set \(S\subseteq N\) with \(|S|=k\); points \(\{t_\ell\}_{\ell=0}^{n-k}\)
\Ensure \(\Phi^{\mathrm{SII}}(S;x)\)
\State \textbf{Lift once:} \(\tilde x_r \gets [\phi_r(x_r),1]\) for all \(r\in N\)

\For{\(\ell = 0 \ \textbf{to}\ n-k\)} \Comment degree \(\le n-k\) \(\Rightarrow\) \(n{-}k{+}1\) points
  \State \(\tilde x_{r\notin S}^{(\ell)} \gets S_r(t_\ell)\,\tilde x_r\)
  \State \(q_\ell \gets \sum_{\mathbf{b}\in\{0,1\}^k} (-1)^{k-\|\mathbf{b}\|_1}\;
           g\big( \{\tilde x_{r\notin S}^{(\ell)}\}, \{S_{i}(b_i)\tilde x_i\}_{i\in S}\big)\)
  \Comment \(2^k\) TN forwards per node
\EndFor

\State Build Vandermonde \(V\in\R^{(n-k+1)\times(n-k+1)}\), \(V_{\ell+1,s+1}=t_\ell^{\,s}\)
\State Solve \(V\,c^{(S)}=q\) to obtain \(c^{(S)}_s,\ s=0{:}n-k\)
\State \textbf{Combine:} \(\displaystyle \Phi^{\mathrm{SII}}(S;x)=\sum_{s=0}^{n-k}\beta_s(n,k)\,c^{(S)}_s\)
\State \Return \(\Phi^{\mathrm{SII}}(S;x)\)
\end{algorithmic}
\textbf{Weights:} \(\displaystyle \beta_s(n,k)=\frac{s!\,(n-k-s)!}{(n-k+1)!}\) (standard SII normalization; adjust if a different convention is used).
\\
\textbf{Complexity:} \(2^k (n{-}k{+}1)\) TN forwards \(+\ O((n{-}k{+}1)^2)\) for the solve; each forward is \(O(\mathrm{poly}(\chi))\).
% \textbf{Complexity:} \(O\!\big(2^k (n{-}k{+}1)\,\mathrm{poly}(\chi)\big)\) TN forwards \(+\ O((n{-}k{+}1)^2)\) for the solve.
\end{algorithm}
\paragraph{Computational complexity.}  
\label{higher_order_complexity}
A naive implementation requires $2^k$ evaluations of $g$ for each probe point (one for each inclusion pattern) and $O((n-k+1)^2)$ time for solving the Vandermonde system. The overall complexity is therefore
\[
O\big(2^k (n - k + 1)\, \mathrm{poly}(\chi)\big) + O((n-k+1)^2) = O\big(2^k n\,\mathrm{poly}(\chi) + n^2\big),
\]
polynomial in $n$ and the bond dimension $\chi$, with exponential dependence on $k$ arising from the intrinsic combinatorial structure of the inclusion--exclusion expansion. The signed--toggle simplification described in \S\ref{subsec:signed-toggle} reduces this dependence to $O(1)$ per probe point while preserving the same interpolation-based recovery.

\subsection{Signed–Toggle Simplification}
\label{subsec:signed-toggle}
The inclusion–exclusion probe above naively requires $2^k$ evaluations per subset $S$.
When the surrogate $g$ is multilinear in the lifted inputs $\{\tilde x_i\}$ and the selectors act
linearly on those inputs, the exponential loop collapses to a single contraction.
% \paragraph{Signed–toggle identity.}
By multilinearity, $g(\ldots,a_i{+}b_i,\ldots)=g(\ldots,a_i,\ldots)+g(\ldots,b_i,\ldots)$, hence
\[
g(\ldots,S_i(1)\tilde x_i,\ldots)-g(\ldots,S_i(0)\tilde x_i,\ldots)
= g\big(\ldots,\big(S_i(1){-}S_i(0)\big)\tilde x_i,\ldots\big).
\]
Applying this recursively for all $i\in S$ yields
\begin{equation}
\label{eq:signed-toggle}
\sum_{\mathbf{b}\in\{0,1\}^k}(-1)^{k-\|\mathbf{b}\|_1}\,
g\big(\ldots,\{S_\ell(b_\ell)\tilde x_\ell\}_{\ell\in S},\ldots\big)
\;=\;
g\big(\ldots,\{(S_\ell(1){-}S_\ell(0))\tilde x_\ell\}_{\ell\in S},\ldots\big).
\end{equation}
For the diagonal selectors $S_i(t)=\Diag(t\,I_{d_i-1},\,1)$ used here,
\[
S_i(1)-S_i(0)=\Diag(I_{d_i-1},\,0),
\]
which zeroes the bias channel and keeps only the data channels. Thus each $(S_i(1)-S_i(0))$ acts as a
discrete derivative / projection operator on feature $i$.

\paragraph{Collapsed probe and complexity.}
With Equation~\ref{eq:signed-toggle}, the probe for a subset $S$ becomes
\[
Q_S(t;x)\;=\;g\big(\{S_r(t)\tilde x_r\}_{r\notin S},\ \{(S_i(1)-S_i(0))\tilde x_i\}_{i\in S}\big),
\]
reducing the per–probe-point cost from $2^k$ forwards to $1$ forward. Evaluating at $n{-}k{+}1$ points and
solving the Vandermonde system (degree $\le n{-}k$) gives the overall complexity
\[
O\big((n{-}k{+}1)\,\mathrm{poly}(\chi)\big) + O\big((n{-}k{+}1)^2\big)
= O\big(n\,\mathrm{poly}(\chi)+n^2\big).
\]
The interpolation step and the size–based reconstruction
$\Phi^{\mathrm{SII}}(S;x)=\sum_{s=0}^{n-k}\beta_s(n,k)\,c_s^{(S)}(x)$ are unchanged.

\section{Comprehensive Experimental Evaluation}
\label{app: experiment}
We provide full experimental protocols, extended baselines comparisons, and additional ablation studies beyond those in the main paper. This includes detailed hyperparameter settings, convergence criteria, cohort selection procedures, and the complete teacher-student rank sweep analysis. We also present comprehensive runtime-accuracy trade-offs across all UCI benchmarks and analyze the impact of feature map dimensionality on higher-order interaction recovery.
\subsection{Experimental Setup and Protocols}
\label{app:exp-setup}

\paragraph{Hardware and Environment.}
All experiments were conducted on NVIDIA Tesla V100-SXM2-32GB GPUs with CUDA acceleration. We used PyTorch~\citep{paszke2019pytorch} with automatic mixed precision for training and inference. For reproducibility, we fixed random seeds (seed=42 for synthetic experiments, seed=2711 for UCI benchmarks) and recorded complete hardware specifications for each run. The primary compute nodes were Intel Xeon E5-2698 v4 @ 2.20GHz (503GB RAM).

\paragraph{Datasets.}
Table~\ref{tab:dataset_characteristics} summarizes the three UCI~\citep{uci_diabetes,uci_concrete,uci_energy} regression tasks used in our experiments. All features and targets were standardized with scikit-learn’s \texttt{StandardScaler} \citep{pedregosa2011scikit}. For local explanations, we selected cohorts of 100 neighbors via k-NN in standardized feature space around each test instance.

\begin{table}[H]
\centering
\small
\caption{Dataset characteristics used in our experiments.}
\label{tab:dataset_characteristics}
\begin{tabular}{lcccc}
\toprule
Dataset & Task & \# Samples & \# Features & Target \\
\midrule
Diabetes & Regression & 442 & 10 & Disease progression \\
Concrete & Regression & 1{,}030 & 8 & Compressive strength \\
Energy (Y1) & Regression & 768 & 8 & Heating load \\
\bottomrule
\end{tabular}
\end{table}

\paragraph{Teacher Models.}
We trained 3-layer MLPs~\citep{rumelhart1986learning, paszke2019pytorch} with architecture [input → 256 → 256 → 128 → 1] using ReLU activations. Teachers were trained until validation $R^2 \geq 0.95$ or 500 epochs, using Adam optimizer with learning rate $10^{-3}$ and early stopping (patience=50).

\paragraph{TN Surrogate Configuration.}
Each surrogate is modeled as a binary tensor tree~\citep{grasedyck2010hierarchical} with bond dimension $\chi=16$. The feature maps are learned through single-hidden-layer MLPs with 64 ReLU~\citep{nair2010rectified} units, producing scalar outputs that are concatenated with a bias term. For every test instance, the training data consist of two parts: a Gaussian neighborhood—using either $M=100$ or the number of test samples, whichever is smaller, with standard deviation $\sigma=0.1\times\text{std}(X_{\text{train}})$—and $2n^2$ structured selector-weighted probes placed at Chebyshev–Gauss nodes~\citep{trefethen2013approximation}. Optimization is performed using Adam with learning rate $10^{-3}$ for up to 1,500 epochs, applying early stopping based on validation $R^2$. A single surrogate is trained for each 100-point cohort to amortize computation across similar instances.

\paragraph{Baseline Methods.}
We benchmarked \methodname{} against established Shapley and interaction estimators under matched query budgets. All baselines were implemented using the \textsc{SHAP-IQ} library \citep{shapiq}, and we included all baseline methods available in that framework: KernelSHAPIQ \citep{fumagalli2024kernelshap}, a regression-based variant of the Shapley Interaction Index (SII); PermutationSampling \citep{tsai2023faith}, a traditional mean-estimation approach; RegressionFSII \citep{tsai2023faith}, which performs feature selection via regression; and SHAPIQ Monte Carlo \citep{shapiq}, a Monte Carlo estimator of SII values. All methods were evaluated at budgets $B \in \{50, 100, 500, 1000, 2000, 10000\}$ model evaluations.

\paragraph{Evaluation Protocol.}
For each test instance: (1) compute exact Shapley values/interactions via exhaustive enumeration (ground truth), (2) apply TN-SHAP using the fitted local surrogate, (3) run all baselines with specified budgets, (4) measure cosine similarity and MSE relative to ground truth, (5) record wall-clock time with CUDA synchronization. Results were aggregated across test sets with mean ± std reported.
We evaluated TN-SHAP on three UCI regression benchmarks, comparing against sampling-based baselines across interaction orders $k \in \{1,2,3\}$. Tables~\ref{tab:diabetes_k1_noR2}--\ref{tab:diabetes_k3_noR2} present comprehensive results for the Diabetes dataset, showing that TN-SHAP achieves comparable or superior accuracy while operating at millisecond timescales.

\subsection{Synthetic Validation Experiments}
\label{app:teacher-student-rank-sweep}

We evaluate how well tensor network (TN) students of varying ranks recover Shapley values and higher-order interactions from known target functions. Each teacher defines a multilinear mapping \(f:\mathbb{R}^4 \to \mathbb{R}\), and students are trained to fit \(f\) and reproduce its Shapley decomposition up to order \(k=3\) using our \methodname\ interpolation procedure (thin-diagonal paths, degree-\(d\) fits from \(d{+}1\) Chebyshev nodes, and inclusion--exclusion for \(k\!\ge\!2\)).

We consider three teacher targets of increasing complexity:
\begin{itemize}[leftmargin=*,itemsep=2pt,topsep=2pt]
\item \text{TensorTree (rank 3):} Balanced binary TN with physical dimensions \([2,2,2,2]\).
\item \text{TensorTree (rank 16):} Higher-capacity variant.
\item \text{Generic Multilinear:} Explicit multilinear function up to order~3 with sparsity \(0.3\).
\end{itemize}

Student TN ranks are \(\{2,4,6,8,10,16\}\).
For each teacher--rank pair we train on \(10{,}000\) i.i.d.\ inputs \(x\!\sim\!\mathcal{N}(0,I_4)\) for up to \(1{,}500\) epochs using Adam (learning rate \(10^{-3}\)) with a ReduceLROnPlateau scheduler; early stopping employs patience \(=200\) with stringent loss/\(R^2\) thresholds.
We run \(\mathbf{10}\) seeds per configuration; seeds are drawn deterministically from an RNG initialized with \texttt{12345}.

Immediately after training each student, we compute TNShap interactions for orders \(k\in\{1,2,3\}\) on \(\mathbf{128}\) random test inputs and report an \emph{aggregated} \(R^2\) by stacking all interaction values across inputs (one score per order).
We summarize results as mean \(\pm\) std across the 10 seeds.

Experiments ran on a Tesla V100\textendash SXM2\textendash 32GB GPU; CUDA acceleration was enabled for training and batched evaluations.

\begin{table}[ht]
\centering
\caption{Aggregated Teacher--Student Rank Sweep ($R^2$ mean $\pm$ std across 10 seeds, 128 test points). 
Results shown for TNShap with TensorTree and Generic Multilinear teachers.}
\label{tab:teacher_student_rank_sweep}
\resizebox{\textwidth}{!}{
\begin{tabular}{llcccccc}
\toprule
Teacher & Student Rank & Train $R^2$ & Order-1 $R^2$ & Order-2 $R^2$ & Order-3 $R^2$ \\
\midrule
\textbf{TensorTree Rank 3} & 2  & 0.8068 ± 0.3048 & 0.7853 ± 0.3503 & 0.8652 ± 0.2797 & 0.8428 ± 0.4298 \\
                           & 4  & 0.9970 ± 0.0040 & 0.9948 ± 0.0056 & 0.9994 ± 0.0008 & 0.9970 ± 0.0057 \\
                           & 6  & 0.9999 ± 0.0000 & 0.9992 ± 0.0003 & 1.0000 ± 0.0000 & 0.9993 ± 0.0003 \\
                           & 8  & 1.0000 ± 0.0000 & 0.9988 ± 0.0007 & 1.0000 ± 0.0000 & 0.9993 ± 0.0004 \\
                           & 10 & 1.0000 ± 0.0000 & 0.9992 ± 0.0006 & 1.0000 ± 0.0000 & 0.9994 ± 0.0004 \\
                           & 16 & 1.0000 ± 0.0000 & 0.9996 ± 0.0002 & 1.0000 ± 0.0000 & 0.9997 ± 0.0001 \\
\midrule
\textbf{TensorTree Rank 16} & 2  & 0.7863 ± 0.2833 & 0.8034 ± 0.3037 & 0.7920 ± 0.2856 & 0.8726 ± 0.3516 \\
                            & 4  & 0.9978 ± 0.0056 & 0.9947 ± 0.0127 & 0.9990 ± 0.0028 & 0.9996 ± 0.0003 \\
                            & 6  & 1.0000 ± 0.0000 & 0.9994 ± 0.0002 & 1.0000 ± 0.0000 & 0.9998 ± 0.0001 \\
                            & 8  & 1.0000 ± 0.0000 & 0.9993 ± 0.0002 & 1.0000 ± 0.0000 & 0.9998 ± 0.0000 \\
                            & 10 & 1.0000 ± 0.0000 & 0.9994 ± 0.0004 & 1.0000 ± 0.0000 & 0.9998 ± 0.0001 \\
                            & 16 & 1.0000 ± 0.0000 & 0.9996 ± 0.0001 & 1.0000 ± 0.0000 & 0.9999 ± 0.0000 \\
\midrule
\textbf{Generic Multilinear} & 2  & 0.8751 ± 0.1457 & --5.5812 ± 9.6013 & 0.7849 ± 0.2398 & --0.0658 ± 1.9349 \\
                             & 4  & 0.9996 ± 0.0008 & 0.9118 ± 0.0539 & 0.9991 ± 0.0019 & 0.9821 ± 0.0109 \\
                             & 6  & 1.0000 ± 0.0000 & 0.9410 ± 0.0092 & 1.0000 ± 0.0001 & 0.9861 ± 0.0023 \\
                             & 8  & 1.0000 ± 0.0000 & 0.9396 ± 0.0116 & 0.9999 ± 0.0001 & 0.9829 ± 0.0039 \\
                             & 10 & 1.0000 ± 0.0000 & 0.9430 ± 0.0091 & 1.0000 ± 0.0001 & 0.9832 ± 0.0031 \\
                             & 16 & 1.0000 ± 0.0000 & 0.9514 ± 0.0053 & 1.0000 ± 0.0001 & 0.9854 ± 0.0014 \\
\bottomrule
\end{tabular}
}
\end{table}

Table~\ref{tab:teacher_student_rank_sweep} reports the aggregated $R^2$ between teacher and student TNShap interaction tensors across ten random seeds. 
For both TensorTree teachers, performance improved rapidly with rank: 
low-rank students (rank~2) achieved moderate $R^2$ ($\approx0.8$) across orders, while rank~4--6 students nearly matched the teachers ($R^2>0.99$). 
Beyond rank~6, all interaction orders reached near-perfect agreement ($R^2>0.999$), confirming that the student TNs can reliably recover the underlying Shapley structure once expressive capacity suffices.
For the Generic~Multilinear teacher, rank~2 students underfit, yielding unstable or even negative $R^2$, but ranks~$\ge4$ stabilized ($R^2_{k=1}\!\approx\!0.94$, $R^2_{k=3}\!\approx\!0.98$), demonstrating that modest tensor ranks can already capture complex multilinear dependencies.
Overall, these results validate that the proposed \methodname{} approach faithfully recovers first-, second-, and third-order Shapley interactions when the student TN rank matches the intrinsic multilinearity of the target.

We performed a focused rerun for the \emph{Generic Multilinear} teacher with a rank-2 student across \textbf{30} random seeds, reporting aggregated $R^2$ (stacked across test inputs) as mean~$\pm$~std:
\[
\text{Train }R^2 = 0.9592 \pm 0.0792,\qquad
R^2_{k=1} = -0.4547 \pm 2.4141,\qquad
R^2_{k=2} = 0.9245 \pm 0.1717,\qquad
R^2_{k=3} = 0.8186 \pm 0.4037.
\]
The negative and highly variable $R^2$ for $k{=}1$ is \emph{expected} in this setting and reflects metric instability rather than a failure of TNShap. In the Generic Multilinear teacher, most of the signal resides in higher-order terms; consequently, the \emph{first-order} Shapley values $\{\phi_i\}$ have \emph{very small variance} across features and test points. Since $R^2 \!=\! 1 - \mathrm{MSE}/\mathrm{Var}(y)$, a near-zero $\mathrm{Var}(\phi_i)$ makes $R^2$ ill-conditioned: even small absolute errors inflate the ratio and can yield large negative values. In contrast, the \emph{pairwise} and \emph{three-way} interactions ($k{=}2,3$) exhibit larger variance and are therefore stably recovered by the student (high positive $R^2$). Practically, this indicates that with rank~2 the model underfits main effects while still capturing a substantial portion of higher-order structure.
Figure~\ref{fig:heatmap_gt14} visualizes the safe-$R^2$ scores for different student ranks when the ground-truth tensor network has rank~14.
Each cell corresponds to the reconstruction quality of order-$k$ interactions for a given student rank.
Performance improves monotonically with rank, saturating near~1.0 once the student rank exceeds the ground-truth rank.
Lower-rank students underestimate higher-order interactions ($k\!\ge\!3$), confirming that sufficient tensor rank is crucial for capturing complex coalitional structures.
\begin{figure}[H]
    \centering
    \includegraphics[width=0.6\linewidth]{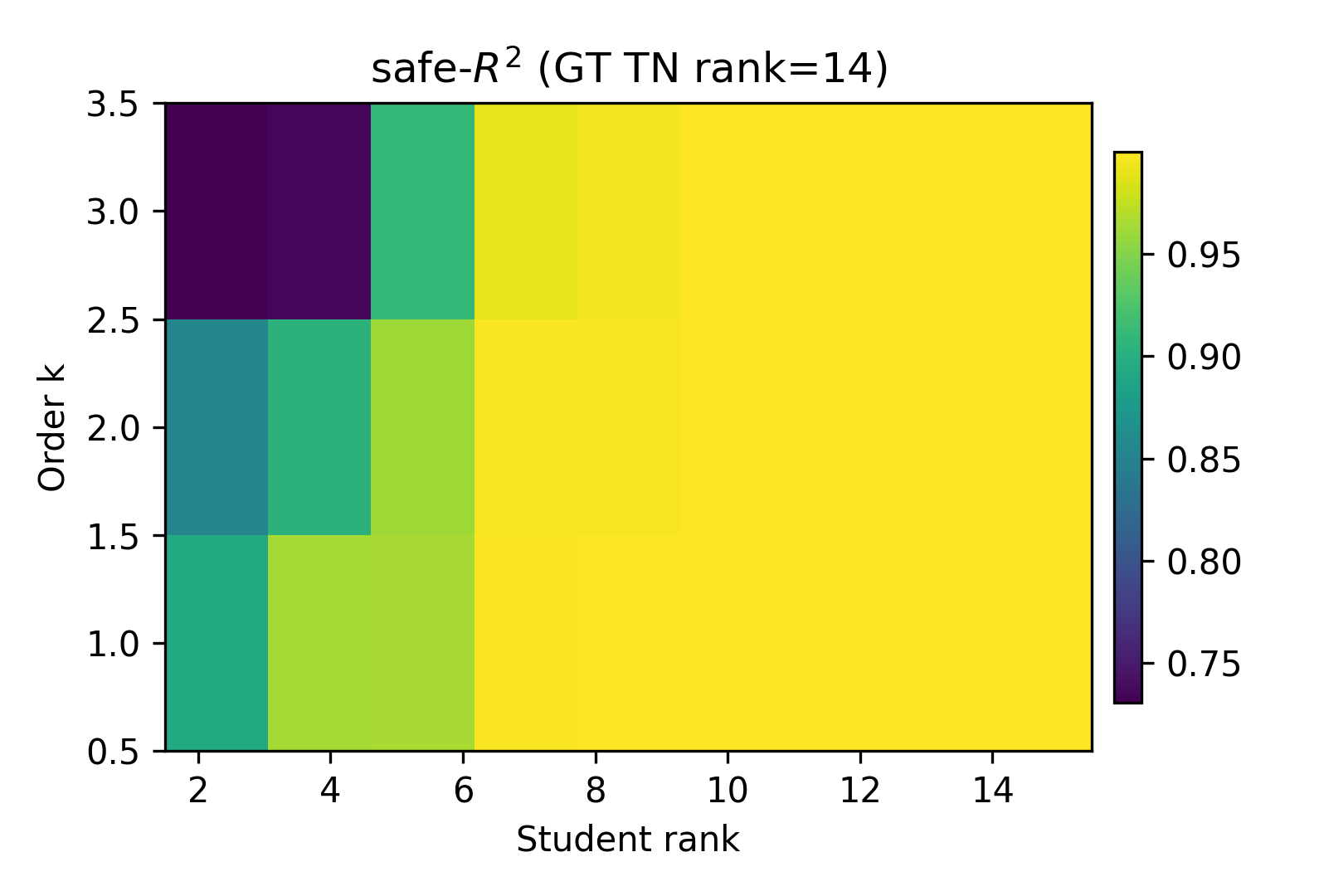}
    \caption{Rank ablation heatmap: safe-$R^2$ for student ranks under a ground-truth tensor network of rank~14.}
    \label{fig:heatmap_gt14}
\end{figure}
\paragraph{Scaling Experiments.}
For dimensions $d \in \{10,20,30,40,50\}$, we generated CP-decomposed multilinear functions and fitted rank-16 tensor tree students. Each configuration used 10 test points with exact ground truth for runtime comparison against KernelSHAP-IQ. TN-SHAP achieved consistent millisecond-scale attribution across all dimensions, with runtime scaling approximately linearly in $d$.

\subsection{UCI Benchmark Extended Results}
We benchmark \methodname on the UCI \textsc{Diabetes} regression task against standard Shapley/interaction estimators  KernelSHAPIQ (Regression SII)~\citep{fumagalli2024kernelshap}, RegressionFSII~\citep{tsai2023faith}, and PermutationSamplingSII~\citep{tsai2023faith} , computing first-, second-, and third-order values under matched budgets and identical input preprocessing. The TN surrogate is a feature-mapped tensor network trained with GPU-accelerated contractions, enabling fast estimation of contributions and interactions while keeping the evaluation protocol identical across methods. We follow the experiment set up in Section~\ref{app:exp-setup}.
The results in Tables~\ref{tab:diabetes_k1_noR2}--\ref{tab:diabetes_k3_noR2} compare the performance of \methodname{} with several baseline Shapley estimation methods across increasing coalition orders $k=1,2,3$ on the \texttt{Diabetes} dataset, $(+0.0654)$ denotes the amortized training time (in seconds per test point) added to the per-point evaluation time for TN-SHAP. All other methods report per-point evaluation time only. For $k=1$, \methodname{} achieves near-perfect cosine similarity ($0.9937\pm0.0060$) and extremely low mean-squared error ($4.9\times10^{-4}$), outperforming all other baselines in accuracy and almost equal runtime to very low budget sampling methods (0.07~s versus 0.06--12~s). This confirms that the tensor network surrogate captures first-order (individual-feature) contributions efficiently and faithfully. 
At $k=2$, the approximation difficulty increases: while \methodname{} maintains a low MSE ($3.8\times10^{-4}$) and fast inference, its cosine similarity ($0.64\pm0.19$) reflects the inherent challenge of reconstructing pairwise interactions with a compact multilinear model. Nevertheless, \methodname{} performs comparably or better than regression- and sampling-based methods, which require orders of magnitude more runtime. 
For $k=3$, higher-order interaction estimation becomes unstable for all baselines, cosine similarity drops below~0.3 and MSE values remain small but noisy, yet \methodname{} still retains competitive accuracy with the lowest computational cost. Overall, these results demonstrate that the tensor network approach provides an excellent trade-off between efficiency and fidelity, particularly at low and moderate interaction orders, validating its scalability as a surrogate for computing coalitional Shapley indices.

% =============== TABLE: k = 1 (no R^2) ===============
\begin{table}[t]
\centering
\caption{Baseline Performance (Diabetes) — Order $k=1$}
\label{tab:diabetes_k1_noR2}
\resizebox{\textwidth}{!}{
\begin{tabular}{llccccc}
\toprule
Baseline & Order & Budget & Targets & Time (s) & Cosine Sim. & MSE \\
\midrule
\textbf{TN-SHAP} & $k=1$ & -- & 89 &\textbf{0.0069 ± 0.0323} (+0.0654) & 0.9937 ± 0.0060 & 0.000499 ± 0.000397 \\
\midrule

KernelSHAPIQ & $k=1$ & 50.0 & 89 & 0.0628 ± 0.0102 & 0.9827 ± 0.0109 & 0.001159 ± 0.000435 \\
KernelSHAPIQ & $k=1$ & 100.0 & 89 & 0.0898 ± 0.0177 & 0.9847 ± 0.0125 & 0.001124 ± 0.000436 \\
KernelSHAPIQ & $k=1$ & 500.0 & 89 & 0.2846 ± 0.0523 & 0.9910 ± 0.0086 & 0.000675 ± 0.000329 \\
KernelSHAPIQ & $k=1$ & 1000.0 & 89 & 0.4787 ± 0.0194 & 0.9863 ± 0.0114 & 0.001011 ± 0.000388 \\
KernelSHAPIQ & $k=1$ & 2000.0 & 89 & 0.4550 ± 0.0267 & 0.9923 ± 0.0064 & 0.000546 ± 0.000236 \\
KernelSHAPIQ & $k=1$ & 10000.0 & 89 & 0.4598 ± 0.0454 & 0.9921 ± 0.0071 & 0.000585 ± 0.000274 \\
\midrule
PermutationSamp & $k=1$ & 50.0 & 89 & 0.1103 ± 0.0135 & 0.9897 ± 0.0082 & 0.000715 ± 0.000307 \\
PermutationSamp & $k=1$ & 100.0 & 89 & 0.1653 ± 0.0341 & 0.9853 ± 0.0115 & 0.001143 ± 0.000423 \\
PermutationSamp & $k=1$ & 500.0 & 89 & 0.6640 ± 0.0935 & 0.9913 ± 0.0082 & 0.000687 ± 0.000355 \\
PermutationSamp & $k=1$ & 1000.0 & 89 & 1.2574 ± 0.1079 & 0.9857 ± 0.0118 & 0.001042 ± 0.000401 \\
PermutationSamp & $k=1$ & 2000.0 & 89 & 2.5091 ± 0.2002 & 0.9921 ± 0.0065 & 0.000546 ± 0.000222 \\
PermutationSamp & $k=1$ & 10000.0 & 89 & 12.2420 ± 0.7666 & 0.9921 ± 0.0071 & 0.000585 ± 0.000272 \\
\midrule
RegressionFSII & $k=1$ & 50.0 & 89 & 0.0606 ± 0.0069 & 0.9827 ± 0.0109 & 0.001159 ± 0.000435 \\
RegressionFSII & $k=1$ & 100.0 & 89 & 0.0875 ± 0.0111 & 0.9847 ± 0.0125 & 0.001124 ± 0.000436 \\
RegressionFSII & $k=1$ & 500.0 & 89 & 0.2793 ± 0.0289 & 0.9910 ± 0.0086 & 0.000675 ± 0.000329 \\
RegressionFSII & $k=1$ & 1000.0 & 89 & 0.4776 ± 0.0167 & 0.9863 ± 0.0114 & 0.001011 ± 0.000388 \\
RegressionFSII & $k=1$ & 2000.0 & 89 & 0.4536 ± 0.0149 & 0.9923 ± 0.0064 & 0.000546 ± 0.000236 \\
RegressionFSII & $k=1$ & 10000.0 & 89 & 0.4567 ± 0.0440 & 0.9921 ± 0.0071 & 0.000585 ± 0.000274 \\
\midrule
SHAPIQ (MonteCarlo SII) & $k=1$ & 50.0 & 89 & 0.0610 ± 0.0059 & 0.8446 ± 0.0893 & 0.013974 ± 0.013453 \\
SHAPIQ (MonteCarlo SII) & $k=1$ & 100.0 & 89 & 0.0876 ± 0.0069 & 0.9120 ± 0.0684 & 0.008089 ± 0.007760 \\
SHAPIQ (MonteCarlo SII) & $k=1$ & 500.0 & 89 & 0.2751 ± 0.0265 & 0.9863 ± 0.0103 & 0.000989 ± 0.000387 \\
SHAPIQ (MonteCarlo SII) & $k=1$ & 1000.0 & 89 & 0.4850 ± 0.0339 & 0.9864 ± 0.0112 & 0.001011 ± 0.000398 \\
SHAPIQ (MonteCarlo SII) & $k=1$ & 2000.0 & 89 & 0.4575 ± 0.0276 & 0.9923 ± 0.0064 & 0.000546 ± 0.000236 \\
SHAPIQ (MonteCarlo SII) & $k=1$ & 10000.0 & 89 & 0.4616 ± 0.0703 & 0.9921 ± 0.0071 & 0.000585 ± 0.000274 \\
\bottomrule
\end{tabular}
}
\end{table}

% =============== TABLE: k = 2 (no R^2) ===============
\begin{table}[t]
\centering
\caption{Baseline Performance (Diabetes) — Order $k=2$}
\label{tab:diabetes_k2_noR2}
\resizebox{\textwidth}{!}{
\begin{tabular}{llccccc}
\toprule
Baseline & Order & Budget & Targets & Time (s) & Cosine Sim. & MSE \\
\midrule
\textbf{TN-SHAP} & $k=2$ & -- & 89 & \textbf{0.0058 ± 0.0097} (+0.0654)  & 0.6435 ± 0.1894 & 0.000377 ± 0.000339 \\
\midrule
KernelSHAPIQ & $k=2$ & 50.0 & 89 & 0.0599 ± 0.0086 & 0.0909 ± 0.1035 & 0.126899 ± 0.123901 \\
KernelSHAPIQ & $k=2$ & 100.0 & 89 & 0.0891 ± 0.0136 & 0.4129 ± 0.1538 & 0.001030 ± 0.000212 \\
KernelSHAPIQ & $k=2$ & 500.0 & 89 & 0.2772 ± 0.0263 & 0.6505 ± 0.1831 & 0.000305 ± 0.000120 \\
KernelSHAPIQ & $k=2$ & 1000.0 & 89 & 0.4834 ± 0.0184 & 0.6462 ± 0.1761 & 0.000318 ± 0.000124 \\
KernelSHAPIQ & $k=2$ & 2000.0 & 89 & 0.4621 ± 0.0418 & 0.7204 ± 0.1681 & 0.000241 ± 0.000119 \\
KernelSHAPIQ & $k=2$ & 10000.0 & 89 & 0.4656 ± 0.0761 & 0.6787 ± 0.1834 & 0.000277 ± 0.000128 \\
\midrule
PermutationSamp & $k=2$ & 50.0 & 89 & 0.0265 ± 0.0039 & 0.0000 ± 0.0000 & 0.000595 ± 0.000436 \\
PermutationSamp & $k=2$ & 100.0 & 89 & 0.0857 ± 0.0141 & 0.1560 ± 0.0966 & 0.000630 ± 0.000376 \\
PermutationSamp & $k=2$ & 500.0 & 89 & 0.2996 ± 0.0375 & 0.4324 ± 0.1572 & 0.000604 ± 0.000232 \\
PermutationSamp & $k=2$ & 1000.0 & 89 & 0.5874 ± 0.0487 & 0.4694 ± 0.1502 & 0.000811 ± 0.000161 \\
PermutationSamp & $k=2$ & 2000.0 & 89 & 1.1447 ± 0.1233 & 0.6941 ± 0.1509 & 0.000273 ± 0.000130 \\
PermutationSamp & $k=2$ & 10000.0 & 89 & 5.6492 ± 0.3430 & 0.6683 ± 0.1823 & 0.000286 ± 0.000122 \\
\midrule
RegressionFSII & $k=2$ & 50.0 & 89 & 0.0585 ± 0.0040 & 0.0785 ± 0.0959 & 0.233139 ± 0.453246 \\
RegressionFSII & $k=2$ & 100.0 & 89 & 0.0867 ± 0.0061 & 0.5747 ± 0.1642 & 0.000443 ± 0.000150 \\
RegressionFSII & $k=2$ & 500.0 & 89 & 0.2765 ± 0.0160 & 0.6923 ± 0.1828 & 0.000263 ± 0.000127 \\
RegressionFSII & $k=2$ & 1000.0 & 89 & 0.4774 ± 0.0196 & 0.6892 ± 0.1814 & 0.000268 ± 0.000128 \\
RegressionFSII & $k=2$ & 2000.0 & 89 & 0.4549 ± 0.0232 & 0.7342 ± 0.1687 & 0.000230 ± 0.000121 \\
RegressionFSII & $k=2$ & 10000.0 & 89 & 0.4554 ± 0.0225 & 0.6916 ± 0.1866 & 0.000264 ± 0.000133 \\
\midrule
SHAPIQ (MonteCarlo SII) & $k=2$ & 50.0 & 89 & 0.0619 ± 0.0107 & 0.0369 ± 0.1007 & 0.131029 ± 0.135230 \\
SHAPIQ (MonteCarlo SII) & $k=2$ & 100.0 & 89 & 0.0911 ± 0.0152 & 0.0624 ± 0.1233 & 0.043445 ± 0.039803 \\
SHAPIQ (MonteCarlo SII) & $k=2$ & 500.0 & 89 & 0.2819 ± 0.0403 & 0.2649 ± 0.1200 & 0.002771 ± 0.002456 \\
SHAPIQ (MonteCarlo SII) & $k=2$ & 1000.0 & 89 & 0.4899 ± 0.0220 & 0.6123 ± 0.1598 & 0.000365 ± 0.000139 \\
SHAPIQ (MonteCarlo SII) & $k=2$ & 2000.0 & 89 & 0.4621 ± 0.0150 & 0.7204 ± 0.1681 & 0.000241 ± 0.000119 \\
SHAPIQ (MonteCarlo SII) & $k=2$ & 10000.0 & 89 & 0.4619 ± 0.0207 & 0.6787 ± 0.1834 & 0.000277 ± 0.000128 \\
\bottomrule
\end{tabular}
}
\end{table}

% =============== TABLE: k = 3 (no R^2) ===============
\begin{table}[t]
\centering
\caption{Baseline Performance (Diabetes) — Order $k=3$}
\label{tab:diabetes_k3_noR2}
\resizebox{\textwidth}{!}{
\begin{tabular}{llccccc}
\toprule
Baseline & Order & Budget & Targets & Time (s) & Cosine Sim. & MSE \\
\midrule
\textbf{TN-SHAP} & $k=3$ & -- & 89 & \textbf{0.0052 ± 0.0002} (+0.0654) & 0.1433 ± 0.4013 & 0.000184 ± 0.000240 \\
\midrule
KernelSHAPIQ & $k=3$ & 50.0 & 89 & 0.0964 ± 0.0099 & -0.0060 ± 0.0731 & 0.000186 ± 0.000131 \\
KernelSHAPIQ & $k=3$ & 100.0 & 89 & 0.1645 ± 0.0319 & 0.0568 ± 0.0929 & 0.000248 ± 0.000101 \\
KernelSHAPIQ & $k=3$ & 500.0 & 89 & 0.3584 ± 0.0139 & 0.1476 ± 0.1919 & 0.000378 ± 0.000075 \\
KernelSHAPIQ & $k=3$ & 1000.0 & 89 & 0.5719 ± 0.0827 & 0.2009 ± 0.2349 & 0.000212 ± 0.000066 \\
KernelSHAPIQ & $k=3$ & 2000.0 & 89 & 0.5491 ± 0.0693 & 0.2730 ± 0.3015 & 0.000127 ± 0.000064 \\
KernelSHAPIQ & $k=3$ & 10000.0 & 89 & 0.5436 ± 0.0339 & 0.2252 ± 0.2796 & 0.000150 ± 0.000070 \\
\midrule
PermutationSamp & $k=3$ & 50.0 & 89 & 0.0264 ± 0.0038 & 0.0000 ± 0.0000 & 0.000137 ± 0.000120 \\
PermutationSamp & $k=3$ & 100.0 & 89 & 0.0265 ± 0.0042 & 0.0000 ± 0.0000 & 0.000137 ± 0.000120 \\
PermutationSamp & $k=3$ & 500.0 & 89 & 0.2706 ± 0.0470 & 0.0693 ± 0.0799 & 0.000224 ± 0.000117 \\
PermutationSamp & $k=3$ & 1000.0 & 89 & 0.5102 ± 0.0702 & 0.0786 ± 0.1185 & 0.000386 ± 0.000103 \\
PermutationSamp & $k=3$ & 2000.0 & 89 & 0.9893 ± 0.1018 & 0.1530 ± 0.1659 & 0.000185 ± 0.000101 \\
PermutationSamp & $k=3$ & 10000.0 & 89 & 4.7880 ± 0.2098 & 0.1706 ± 0.2349 & 0.000209 ± 0.000076 \\
\midrule
RegressionFSII & $k=3$ & 50.0 & 89 & 0.0587 ± 0.0038 & 0.0080 ± 0.0496 & 0.249489 ± 0.376970 \\
RegressionFSII & $k=3$ & 100.0 & 89 & 0.0917 ± 0.0069 & -0.0040 ± 0.0579 & 0.887906 ± 0.275777 \\
RegressionFSII & $k=3$ & 500.0 & 89 & 0.2769 ± 0.0266 & 0.2364 ± 0.2925 & 0.000142 ± 0.000068 \\
RegressionFSII & $k=3$ & 1000.0 & 89 & 0.4805 ± 0.0153 & 0.2579 ± 0.3101 & 0.000133 ± 0.000067 \\
RegressionFSII & $k=3$ & 2000.0 & 89 & 0.4558 ± 0.0135 & 0.3096 ± 0.3361 & 0.000115 ± 0.000063 \\
RegressionFSII & $k=3$ & 10000.0 & 89 & 0.4641 ± 0.0354 & 0.2584 ± 0.3354 & 0.000128 ± 0.000071 \\
\midrule
SHAPIQ (MonteCarlo SII) & $k=3$ & 50.0 & 89 & 0.0669 ± 0.0074 & 0.0059 ± 0.0578 & 0.051156 ± 0.053770 \\
SHAPIQ (MonteCarlo SII) & $k=3$ & 100.0 & 89 & 0.1011 ± 0.0177 & 0.0114 ± 0.0544 & 0.010474 ± 0.010887 \\
SHAPIQ (MonteCarlo SII) & $k=3$ & 500.0 & 89 & 0.2937 ± 0.0439 & 0.0183 ± 0.0473 & 0.017592 ± 0.016229 \\
SHAPIQ (MonteCarlo SII) & $k=3$ & 1000.0 & 89 & 0.5078 ± 0.0225 & 0.1013 ± 0.1486 & 0.000557 ± 0.000361 \\
SHAPIQ (MonteCarlo SII) & $k=3$ & 2000.0 & 89 & 0.4818 ± 0.0150 & 0.2730 ± 0.3015 & 0.000127 ± 0.000064 \\
SHAPIQ (MonteCarlo SII) & $k=3$ & 10000.0 & 89 & 0.4832 ± 0.0344 & 0.2252 ± 0.2796 & 0.000150 ± 0.000070 \\
\bottomrule
\end{tabular}
}
\end{table}
\subsection{Feature Map Dimensionality Study}
\label{app: feature_map_exps}
We study how the dimensionality of the per--feature nonlinear map
$\psi : \mathbb{R} \to \mathbb{R}^{m_{\text{fmap}}}$ 
affects the quality of the learned multilinear surrogate.
Each feature is first passed through a small MLP
producing $m_{\text{fmap}}$ output channels
with $\psi(0)=0$, and the resulting vector
$\bigl[\psi(x_i),1\bigr]$ forms the physical leg
of each tensor--network (TN) leaf.
We vary $m_{\text{fmap}} \in \{1,2,4,8\}$,
train a TN--based student on the dataset augmented with sampled interpolation points, generated from the \texttt{diabetes} regression benchmark,
and evaluate its recovered interaction tensors
against an exact teacher model for orders $k=1,2,3$ using cosine similarity and mean--squared error (MSE).
\begin{table}[H]
\centering
\small
\setlength{\tabcolsep}{6pt}
\renewcommand{\arraystretch}{1.15}
\begin{tabular}{ccccccc}
\toprule
\multirow{2}{*}{$m_{\text{fmap}}$} &
\multicolumn{2}{c}{Order $k=1$} &
\multicolumn{2}{c}{Order $k=2$} &
\multicolumn{2}{c}{Order $k=3$} \\
\cmidrule(lr){2-3} \cmidrule(lr){4-5} \cmidrule(lr){6-7}
  & cos & MSE ($\times 10^{-4}$) & cos & MSE ($\times 10^{-4}$) & cos & MSE ($\times 10^{-4}$) \\ \midrule
1 & 0.995 & 4.24 & 0.667 & 3.55 & 0.141 & 2.06 \\
2 & 0.973 & 21.7 & 0.525 & 4.39 & 0.130 & 1.75 \\
4 & 0.997 & 2.37 & 0.789 & 2.82 & 0.303 & 1.74 \\
8 & 0.994 & 6.22 & 0.663 & 3.50 & 0.212 & 1.72 \\
\bottomrule
\end{tabular}
\caption{TN selector vs. Teacher on \texttt{diabetes}: cosine similarity and mean squared error (MSE, scaled by $10^{-4}$) between estimated and exact interaction tensors for different feature--map sizes.}
\label{tab:diabetes_ablation_fmap}
\end{table}
For first--order effects ($k=1$), all feature--map sizes yield near--perfect
agreement with the teacher, confirming that linear attributions are easily
captured. Higher--order interactions ($k=2,3$) benefit from richer embeddings:
increasing $m_{\text{fmap}}$ from $1$ to $4$ notably improves cosine similarity,
suggesting that moderate latent width suffices to model nonlinear
pairwise and triple dependencies, while larger maps give diminishing returns.

We also train a plain tensor network whose leaves receive only the raw augmented inputs \([x_i, 1]\), i.e., without any learned feature map \(\psi\).
This baseline isolates the contribution of the feature map while keeping the FewEval masking, shared Chebyshev grid, teacher model, training procedure, and evaluation metrics identical to the main study.
On \texttt{diabetes}, the plain TN achieves strong first–order recovery (cos \(\approx 0.996\)), moderate agreement for pairwise interactions (cos \(\approx 0.71\)), and weak alignment for third–order terms (cos \(\approx 0.12\)).
These results indicate that the TN’s multilinear structure over \([x_i, 1]\) captures most linear effects and part of the quadratic structure, whereas richer higher–order dependencies benefit from a learned per–feature embedding (e.g., a small–channel \(\psi\)), consistent with our ablation trends.
\subsection{Code availability.}
We provide an implementation of our method and experimental setup at \url{https://github.com/farzana0/TN-SHAP}.

\end{document}